\documentclass[hidelinks,onefignum,onetabnum]{siamart251216}

\usepackage{lipsum}
\usepackage{amsfonts,amssymb}
\usepackage{graphicx}
\usepackage{epstopdf}
\usepackage{mathtools}
\usepackage{graphicx}
\usepackage{subcaption}
\usepackage{comment}

\usepackage[utf8]{inputenc} 
\usepackage[T1]{fontenc}
\usepackage{lmodern}
\usepackage[expansion=false]{microtype}  
\usepackage{hyperref}       
\usepackage{url}            
\usepackage{booktabs}       
\usepackage{nicefrac}       
\usepackage[textsize=tiny,textwidth=2.0cm]{todonotes}
\usepackage{blindtext}
\usepackage{mathtools}
\usepackage{amsmath}
\usepackage{amsfonts}
\usepackage{mathrsfs,bbm}
\usepackage{algorithm}
\usepackage{algorithmic}
\usepackage{comment}
\usepackage{xcolor}
\usepackage[shortlabels]{enumitem}

\ifpdf
  \DeclareGraphicsExtensions{.eps,.pdf,.png,.jpg}
\else
  \DeclareGraphicsExtensions{.eps}
\fi

\newsiamremark{remark}{Remark}
\newsiamremark{hypothesis}{Hypothesis}
\crefname{hypothesis}{Hypothesis}{Hypotheses}
\newsiamthm{claim}{Claim}
\newsiamthm{assumption}{Assumption}
\newsiamremark{fact}{Fact}
\crefname{fact}{Fact}{Facts}

\newcommand{\eps}{\epsilon}
\newcommand{\up}{^\prime}
\newcommand{\upp}{^{\prime\prime}}
\newcommand{\lst}{_\star}
\newcommand{\ust}{^\star}
\newcommand{\tkp}{\hat{\kappa}}
\newcommand{\tLm}{\widehat{\Lambda}}
\newcommand{\tGm}{\widehat{\Gamma}}
\newcommand{\tnu}{\hat{\nu}}
\newcommand{\tmix}{t_{\text{mix}}}

\newcommand{\bA}{\mathbf{A}}
\newcommand{\bb}{\mathbf{b}}

\newcommand{\bE}{\mathbb{E}}
\newcommand{\bN}{\mathbb{N}}
\newcommand{\bP}{\mathbb{P}}

\newcommand{\bR}{\mathbb{R}}

\newcommand{\cB}{\mathcal{B}}

\newcommand{\cF}{\mathcal{F}}

\newcommand{\cL}{\mathcal{L}}

\newcommand{\cO}{\mathcal{O}}

\newcommand{\cT}{\mathcal{T}}

\newcommand{\cZ}{\mathcal{Z}}

\newcommand{\cfr}{\mathfrak{c}}
\newcommand{\mfr}{\mathfrak{m}}

\newcommand{\flbr}[1]{\left\{#1\right\}}
\newcommand{\sqbr}[1]{\left[#1\right]}
\newcommand{\br}[1]{\left(#1\right)}
\newcommand{\angl}[1]{\left\langle #1 \right\rangle}
\newcommand{\abs}[1]{\left|#1\right|}
\newcommand{\norm}[1]{\left\Vert #1 \right\Vert}

\newcommand{\floor}[1]{\left\lfloor#1\right\rfloor}

\newcommand{\ind}[1]{\mathbbm{1}\br{#1}}
\newcommand{\uc}[1]{^{\br{#1}}}
\newcommand{\tr}{\textit{tr}}

\def\rmd{\mathrm{d}}

\headers{SGD under the P\L~Condition with Markovian Noise}{A. Kar, S. Chandak, R. Singh, E. Moulines, S. Bhatnagar, and N. Bambos}

\title{High-Probability Bounds for SGD under the Polyak-\L ojasiewicz Condition with Markovian Noise\thanks{Corresponding Authors: \email{avikkar@iisc.ac.in, chandaks@stanford.edu}}}

\author{Avik Kar\thanks{Department of Computer Science and Automation, Indian Institute of Science, Bengaluru, India.}
\and Siddharth Chandak\thanks{Department of Electrical Engineering, Stanford University, Stanford, CA, USA.}
\and Rahul Singh
\and Eric Moulines\thanks{CMAP, CNRS, \'Ecole polytechnique, Institut Polytechnique \'de Paris, Palaiseau, France.}
\and Shalabh Bhatnagar\footnotemark[2]
\and Nicholas Bambos \footnotemark[3]
}

\usepackage{amsopn}

\ifpdf
\hypersetup{
  pdftitle={High-Probability Bounds for SGD under the P\L~Condition with Markovian Noise},
  pdfauthor={Kar, S. Chandak, R. Singh, E. Moulines, S. Bhatnagar, and N. Bambos}
}
\fi

\begin{document}

\maketitle

\begin{abstract}
    We present the first uniform-in-time high-probability bound for SGD under the PL condition, where the gradient noise contains both Markovian and martingale difference components. This significantly broadens the scope of finite-time guarantees, as the P\L~condition arises in many machine learning and deep learning models while Markovian noise naturally arises in decentralized optimization and online system identification problems. We further allow the magnitude of noise to grow with the function value, enabling the analysis of many practical sampling strategies. In addition to the high-probability guarantee, we establish a matching $1/k$ decay rate for the expected suboptimality. Our proof technique relies on the Poisson equation to handle the Markovian noise and a probabilistic induction argument to address the lack of almost-sure bounds on the objective. Finally, we demonstrate the applicability of our framework by analyzing three practical optimization problems: token-based decentralized linear regression, supervised learning with subsampling for privacy amplification, and online system identification.
\end{abstract}



\section{Introduction}\label{sec:intro}
We consider the stochastic gradient descent (SGD) algorithm for solving the unconstrained optimization problem $\arg\min_{x \in \bR^d} f(x)$. Given a stepsize (learning rate) sequence $\{\alpha_k\}_{k\geq 0}$, the SGD iteration is as follows.
\begin{align*}
 x_{k+1} = x_k - \alpha_k G_k, \quad k = 0, 1, \ldots,
\end{align*}
where $G_k$ is a noisy estimate of $\nabla f(x_k)$. In this work, we analyze the finite-time behaviour of the iterates, in particular the suboptimality gap of the iterates, $f(x_k) - f\ust$, where $f\ust \coloneqq \inf_{x \in \bR^d} {f(x)}$. SGD is widely used to solve optimization problems in machine learning due to its simplicity and scalability in high-dimensional settings. Over the past two decades, it has played a central role in numerous successful machine learning applications. Much of the practical success of modern machine learning can be attributed to SGD and its variants. Although most theoretical analyses of SGD are conducted under the assumption of convexity of the objective, most optimization tasks in machine learning, including neural network training, are non-convex.

Many relevant non-convex optimization problems in machine learning satisfy the P\L~condition \cite{polyak1963gradient, karimi2016linear}. The objective function $f$ is said to satisfy the  P\L~condition with constant $\mu > 0$ (or ``$f$ is $\mu$-P\L'') if
\begin{align}
    \norm{\nabla f(x)}^2 \geq 2\mu \br{f(x) - f\ust}, ~\forall x \in \bR^d. \label{eq:pl_cond}
\end{align}
The P\L~condition guarantees global linear convergence of gradient descent~\cite{karimi2016linear}. Several important problem classes satisfy the P\L~condition. For example, the class of ``strongly convex composed with linear'' functions, where the objective takes the form $f(x) = g(Ax)$ for a strongly convex function $g$ and matrix $A$ satisfies the P\L~condition~\cite{karimi2016linear}. An important problem instance in this category is the least-squares problem (i.e., $f(x) = \|Ax-b\|^2$). The P\L~condition also holds for logistic regression when the gradient descent iterates remain in a compact set, even though the loss function is strictly convex rather than strongly convex \cite{karimi2016linear}. It also arises in control and reinforcement learning problems; for instance, the objective function associated with the Linear Quadratic Regulator (LQR) problem satisfies the P\L~condition~\cite{fazel2018global}. The P\L~condition has also been established in several modern machine learning models. \cite{liu2022loss} showed that wide neural networks satisfy P\L~condition on a set of radius $R$, and $R$ is proportional to the width of the neural network. Similar results hold for wide convolutional neural networks (CNNs) and residual neural networks (ResNet)~\cite{liu2022loss}.

A significant limitation of much existing finite-time analysis of SGD is the assumption that the noise process forms a martingale difference sequence. In many applications, however, samples are generated along a Markov process, leading to temporally correlated noise. Such settings arise naturally in decentralized optimization, privacy-preserving learning, reinforcement learning, and stochastic control.

For example, in token-based algorithms for decentralized optimization~\cite{even23a}, a token carrying the current iterate performs a random walk over a network of agents, and the sequence of nodes visited forms a Markov chain. In differential privacy, certain subsampling mechanisms used for privacy amplification generate minibatches whose evolution can be modeled as a Markov chain~\cite{dong2026privacy, choquette2023amplified, choquette2025near}. Markovian data also appears in online system identification, where observations are generated by a linear dynamical system~\cite{kowshik2021streaming}. Additionally, Markov chain sampling methods are commonly used when direct sampling from a complex target distribution is computationally expensive, and the resulting gradient estimates inherit the temporal dependence of the chain~\cite{sun2018onmarkov}.

\subsection{Our Contributions}
To address these limitations, we obtain both high-probability and expected bounds on the suboptimality gap of SGD iterates when the objective satisfies the P\L~condition and the gradient noise contains both Markovian and martingale difference components. Our main contributions are as follows:
\begin{itemize}
    \item We establish the first uniform-in-time high-probability bound on the suboptimality gap of SGD iterates under the P\L~condition with Markovian gradient noise. In particular, we show that with probability at least $1-\delta$, the suboptimality gap decays as $\cO\big(\tmix^2 \log\left(k/\delta\right)/k\big)$ for all $k \ge 0$, where $t_{mix}$ denotes the mixing time of the underlying Markov chain.
    \item By working under the general ABC condition \cite{khaled2023better}, we allow the gradient noise to grow with both the value of the loss function and the true gradient, enabling the analysis of practical subsampling strategies.
    \item Under this general setup, no almost-sure bound exists for the loss function values, which prevents the direct application of standard concentration inequalities such as the Azuma-Hoeffding bound~\cite{wainwright2019high}. To address this challenge and obtain uniform-in-time guarantees, we introduce a novel technique based on defining a \textit{good event} $E_k$ and inductively establishing high-probability bounds on the loss function. This approach is simple and intuitive, and may be useful in other optimization and stochastic approximation settings.
    
    \item We also show that the expected suboptimality gap is $\cO\!\left(t_{mix}/k\right)$ for all $k \ge 0$. 
    
    \item In \Cref{sec:applications}, we study three optimization problems that fit within our framework: token-based decentralized linear regression, supervised learning with subsampling for privacy amplification, and online system identification.
\end{itemize}

\subsection{Related Works}
SGD was first introduced in \cite{kiefer1952stochastic} as a stochastic approximation~\cite{robbins1951stochastic} scheme. Since then, it has become a central tool in optimization and machine learning. Due to its wide applicability, numerous works have studied its convergence guarantees. The almost sure convergence of SGD iterates to the set of stationary points was first established in \cite{Ljung}. Over the past two decades, significant effort has been devoted to deriving finite-time guarantees for SGD; see, for example, \cite{eric2011non,qian19sgd,wang2021convergence,wang2023onthe}.

The P\L~condition first appeared in \cite{polyak1963gradient}, while the term \emph{Polyak--\L ojasiewicz (P\L) condition} was popularized by \cite{karimi2016linear}, which also presented a systematic study of optimization algorithms under this condition. Most works studying SGD under the P\L~condition derive bounds on the expected suboptimality of the iterates. While \cite{karimi2016linear} assumes bounded variance of the gradient noise, \cite{Li2021asecond,khaled2023better} allow the noise variance to grow with the function value and derive an $\cO(1/k)$ convergence rate. Beyond expected bounds, \cite{madden2024high} obtains a high-probability bound when the noise sequence is sub-Weibull with bounded variance, and \cite{karandikar2024convergence} derives an asymptotic almost sure convergence rate.

In all the works discussed above, the gradient estimate is assumed to be unbiased, i.e., $\bE\sqbr{G_k \mid x_\ell, \ell \le k} = \nabla f(x_k)$. To the best of our knowledge, existing works studying SGD under Markov noise derive bounds only on the expected suboptimality; see \cite{sun2018onmarkov,doan2020finite,doan2022finite,even23a}. Among these, SGD under the P\L~condition is considered in \cite{doan2022finite, even23a}. 

In contrast to the above works, we derive both high-probability and in-expectation bounds on the suboptimality of SGD iterates when the objective satisfies the $\mu$-P\L~condition and the noise contains both Markovian and martingale difference components.

\subsection{Outline and Notation}
The paper is organized as follows. Section~\ref{sec:not&assum} introduces the assumptions required for our results, which are presented in Section~\ref{sec:results}. Section~\ref{sec:proof_outline} provides a proof sketch of the main result. Section~\ref{sec:applications} presents three practical applications of our framework. Finally, Section~\ref{sec:conclusion} concludes the paper and discusses future directions. Detailed proofs are deferred to the appendices. 

Throughout this work, $\|\cdot\|$ denote Euclidean norm and $\langle x_1, x_2\rangle$ denotes the inner product given by $x_1^\top x_2$. When applied on matrices, $\|\cdot\|$ denotes the corresponding operator norm. Other norms used in the paper include the $\ell_\infty$ norm ($\norm{\cdot}_{\infty}$) and the Frobenius norm for matrices ($\|\cdot\|_F$). We use the standard $\cO$ notation: $f(k)=\cO(g(k))$ if there exists a constant $C>0$ such that $|f(k)|\le Cg(k)$.
\section{Problem Formulation}\label{sec:not&assum}
In this section, we set up the notation and state the assumptions for the problem setting. Consider the iteration
\begin{align}\label{iter:sgd}
    x_{k+1} = x_k - \alpha_k G_k,
\end{align}
for $k = 0,1,\ldots$. Here $x_k \in \mathbb{R}^d$ denotes the iterate, and $G_k$ a noisy estimate of the gradient $\nabla f(x_k)$ at instant $k$, where the objective function $f(\cdot)$ satisfies the P\L~condition \eqref{eq:pl_cond}. The noise in the gradient estimate $G_k$ has two components: a Markovian component and a martingale difference component. We now state additional assumptions on the objective function and formalize the class of admissible noise sequences. 

Our first assumption is that the objective function is $L$-smooth, meaning its gradient is $L$-Lipschitz. This is a standard assumption in the optimization literature~\cite{eric2011non,rakhlin2011making} and is essential for establishing convergence guarantees on SGD.

\begin{assumption}\label{assum:smoothness}
    The function $f$ is $L$-smooth, i.e., there exists a constant $L>0$ such that
    \begin{align}
        \norm{\nabla f(x) - \nabla f(y)} \leq L \norm{x - y},~\forall x,y \in \mathbb{R}^d. \label{eq:smoothness}
    \end{align}
\end{assumption}
A result of this assumption is that  $\norm{\nabla f(x)}^2 \leq 2L \br{f(x) - f\ust}$ for all $x \in \bR^d$ (see Lemma~\ref{lem:ub_grad}). This is often referred to as the descent lemma for SGD iterations with constant stepsizes~\cite{nesterov2013introductory}. Hence, if $f$ is both $\mu$-P\L~\eqref{eq:pl_cond} and $L$-smooth~\eqref{eq:smoothness}, then $\mu \leq L$. 

The noisy gradient $G_k$ is defined as
\begin{align}
    G_k \coloneqq g(x_k, Z_k) + M_{k+1},~ \forall k \in \{0,1,\ldots\},
\end{align}
where $g(x_k, Z_k)$ depends on the current iterate $x_k$ and the state of a Markov chain $\{Z_k\}$. $\{M_{k}\}$ forms a martingale difference sequence and each $M_{k+1}$ is allowed to depend on the current iterate $x_k$. Our next assumption formalizes this.
\begin{assumption}\label{assum:martingale}
    Define the filtration $\{\cF_k\}$ where $\cF_k \coloneqq \sigma\br{x_0, M_{\ell}, Z_\ell, \ell\leq k}$, $k \in  \{0,1,\ldots\}$. Then $\{M_{k}\}$ is a martingale difference sequence with respect to $\{\cF_k\}$, i.e., $\bE[|M_k|]<\infty$ and $\bE[M_{k+1} \mid \cF_k]=0$ for all $k \in  \{0,1,\ldots\}$.
\end{assumption}
We next assume the well-known \textit{ABC} condition on noisy gradient samples~\cite{khaled2023better}. 
\begin{assumption}\label{assum:noisy_grad_norm}
    There exist constants $A,B,C \geq 0$ such that for all $k$,
    \begin{align}
        \norm{G_k}^2 \leq A \norm{\nabla f(x_k)}^2 + B \br{f(x_k) - f\ust} + C, \quad a.s.\label{bdd:noisy_grad_norm}
    \end{align}
\end{assumption}
To derive bounds on the expected suboptimality, it is sufficient to assume a bound only on the conditional second moment of the noisy gradient; see, for example, \cite{qian19sgd, khaled2023better, Li2021asecond}. This motivates the relaxed assumption used in our result on expected suboptimality (Assumption \ref{assum:expt_noisy_grad_norm}). However, such a conditional expectation bound is generally insufficient for establishing high-probability guarantees. A standard assumption for high-probability analysis is the almost-sure bound stated above. For instance, \cite{liu22onalmost} assumes a bound of this form to derive asymptotic convergence rates. We emphasize that the function value or its gradient may be unbounded, and consequently, the noise may also be unbounded. The above assumption can be relaxed by instead assuming that $G_k$ is sub-Gaussian, with variance controlled by the same \textit{ABC} structure. Our analysis extends to this relaxed setting without any significant modification of the proof technique. Since our primary focus is on handling Markovian noise, we omit these technical extensions. We formalize the Markov noise in the next assumption.
\begin{assumption}\label{assum:Markov}
    The process $\{Z_k\}$ is a time-homogeneous Markov chain on a compact measurable space $(\cZ,\cB_\cZ)$ with initial distribution $\pi_0$ and transition kernel $p$, i.e., $\bP(Z_0 \in B) = \pi_0(B)$, and for every $k \in \{0,1,\ldots\}$,
    \begin{align*}
        \bP(Z_{k+1} \in B \mid Z_\ell, \ell \leq k) = \bP(Z_{k+1} \in B \mid Z_k) = p(B \mid Z_k),~\forall B \in \cB_\cZ.
    \end{align*}
    The Markov chain has a unique invariant distribution $\pi$. Moreover, it is geometrically ergodic with mixing time $\tmix \in \{1,2,\ldots\}$, i.e., for every $k \geq 0$
    \begin{align}\label{def:tmix}
        \sup_{z \in \cZ}{\|p\uc{k}(\cdot \mid z) - \pi\|_{TV}} \leq 2^{-\floor{k/\tmix}},
    \end{align}
    where $p\uc{k}$ denotes the $k$-step transition kernel induced by $p$. Finally, for every $x\in\bR^d$,
    \begin{align*}
        \nabla f(x) = \int{g(x,z)~ \pi(dz)}.
    \end{align*}
\end{assumption}
Throughout the paper, we use the following definition of total variation norm~\cite{folland2013real}. Let $\xi$ be a signed measure on $(\cZ,\cB_\cZ)$ with Jordan decomposition $\xi=\xi_+-\xi_-$. Then, $\norm{\xi}_{TV} := \abs{\xi}(\cZ)$. The above assumption is stated in a general form to accommodate both finite-state and continuous-state Markov chains. In the finite-state case, irreducibility and aperiodicity imply geometric mixing. In the continuous-state case, $\mathcal{Z}$ is a compact subset of $\mathbb{R}^m$, where geometric mixing can hold under standard regularity conditions. By directly assuming uniform geometric ergodicity in total variation, we avoid imposing separate structural assumptions for discrete and continuous models. The last part of the assumption states that the true gradient is the stationary expectation of the gradient samples generated by the Markov chain. 

The following assumption imposes Lipschitz and boundedness conditions on the gradient samples corrupted by Markovian noise.
\begin{assumption}\label{assum:lipg}
    The map $g(\cdot,z)$ is $L_g$-Lipschitz for every $z \in \cZ$, i.e.,
    \begin{align}
        \norm{g(x,z) - g(y,z)} \leq L_g \norm{x - y},~\forall x,y \in \mathbb{R}^d. \label{eq:lipg}
    \end{align}
    Additionally, 
    \begin{align}
        \norm{g(x,z)}^2 \leq A \norm{\nabla f(x)}^2 + B (f(x) - f\ust) + C, ~\forall x \in \bR^d, z \in \cZ. \label{eq:gbdd}
    \end{align}
\end{assumption}
We use the same $A, B,$ and $C$ here as in Assumption \ref{assum:noisy_grad_norm} for simplicity. 

\section{Results}
\label{sec:results}
In this section, we present the key high-probability and expectation bounds on the suboptimality of the SGD iterates. For convenience, we denote the suboptimality by $\Delta_k$, i.e., $\Delta_k=f(x_k)-f\ust$. We consider the stepsize
$$\alpha_k=\frac{a}{k+K_0},$$
where $a \geq 2/\mu$ and $K_0> 0$. This choice yields the optimal convergence rate of $\mathcal{O}(1/k)$ and is standard in the analysis of stochastic approximation and SGD. Our analysis can also be extended to stepsizes of the form $\Theta(1/k^b)$ for $b\in(0.5,1]$, which yields a convergence rate of $\cO(1/k^b)$. We now present our main result: a high-probability bound on the suboptimality of iterates generated by the SGD iteration \eqref{iter:sgd}.
\begin{theorem}\label{thm:concentration_markov}
    Let the objective function $f(\cdot)$ be $\mu$-P\L~\eqref{eq:pl_cond} and $L$-smooth (Assumption \ref{assum:smoothness}). Suppose the noisy gradient samples satisfy Assumptions \ref{assum:martingale}--\ref{assum:lipg}. Then for every $\delta\in(0,1)$, there exists  $\cfr_1=\Omega\left(\log\left(\frac{1}{\delta}\right)\right)$ such that for $K_0\geq \cfr_1$, with probability at least $1 - \delta$,
    \begin{align*}
        \Delta_k &\leq \cO\br{\frac{\tmix^2 \log{\br{\frac{k}{\delta}}}}{\mu^2 k}}\quad \forall k \in  \{1,2,\ldots\}.
    \end{align*}
\end{theorem}
An outline of the proof for \Cref{thm:concentration_markov} is given in \Cref{sec:proof_outline}. The complete proof, including explicit constants and a fully specified (non-orderwise) version of the result, is presented in Appendix \ref{sec:thm_main}. The above bound matches the optimal $1/k$ decay rate and the $\delta$ dependence obtained for unbiased noise with bounded second moment \cite{madden2024high}. The quadratic dependence on $t_{mix}$ appears because of the decomposition of Markov noise into a martingale difference component, followed by an application of the Azuma-Hoeffding bound. In contrast, the expected suboptimality depends only linearly on $t_{mix}$, leaving open the question of whether the same linear dependence can be achieved in the high-probability bound.

\textbf{High-Probability guarantees in absence of Markov noise:} For completeness, we also present a special case where the Markovian noise is absent, i.e., the noisy gradient samples are given by $G_k=\nabla f(x_k)+M_{k+1}$. This corresponds to the standard martingale-difference noise model widely studied in stochastic optimization. To the best of our knowledge, this is the first work establishing high-probability guarantees for SGD under the P\L ~condition with the $ABC$ noise structure. The closest related work is \cite{madden2024high}; however, unlike \cite{madden2024high}, we allow the noise magnitude to scale with the gradient, and our proof technique is substantially different.
\begin{theorem}\label{thm:concentration_mart}
    Let the objective function $f(\cdot)$ be $\mu$-P\L~\eqref{eq:pl_cond} and $L$-smooth (Assumption \ref{assum:smoothness}). Suppose the noisy gradient samples are free of Markov noise, i.e., $G_k = \nabla f(x_k) + M_{k+1}$ for all $k=0,1,\ldots$, and satisfy Assumptions \ref{assum:martingale} and \ref{assum:noisy_grad_norm}. Then for every $\delta\in(0,1)$, there exists  $\cfr_2=\Omega\left(\log\left(\frac{1}{\delta}\right)\right)$ such that for $K_0\geq \cfr_2$, with probability at least $1 - \delta$,
    \begin{align*}
        \Delta_k &\leq \cO\br{\frac{\log{\br{\frac{k}{\delta}}}}{\mu^2 k}}\quad \forall k \in \{1,2,\ldots\}.
    \end{align*}
\end{theorem}
The proof follows the same outline as Theorem \ref{thm:concentration_markov} and is presented in Appendix \ref{sec:mart_noise_proof}. 

\textbf{Bound in expectation:} We now present a bound on the expected suboptimality of the SGD iterates. As discussed after Assumption \ref{assum:martingale}, we do not need an almost sure bound on the noise to obtain bounds in expectation. Instead, it suffices to impose the following relaxed condition on the conditional second moment of the noisy gradients.
\begin{assumption}\label{assum:expt_noisy_grad_norm}
    There exist constants $A,B,C \geq 0$ such that for all $k$,
    \begin{align}
        \bE\sqbr{\norm{G_k}^2 \mid \cF_k} \leq A \norm{\nabla f(x_k)}^2 + B \br{f(x_k) - f\ust} + C, \mbox{ almost surely.} \label{bdd:expt_noisy_grad_norm}
    \end{align}
\end{assumption}
We have the following result on the expected suboptimality of iterates.
\begin{theorem}\label{thm:expectation_markov}
    Let the objective function $f(\cdot)$ be $\mu$-P\L~\eqref{eq:pl_cond} and $L$-smooth (Assumption \ref{assum:smoothness}). Suppose the noisy gradient samples satisfy Assumptions \ref{assum:martingale}, \ref{assum:expt_noisy_grad_norm}, \ref{assum:Markov}, \ref{assum:lipg}. Then there exists constant $\cfr_3$ such that for $K_0\geq \cfr_3$, 
    \begin{align*}
        \bE\sqbr{\Delta_k} &\leq \cO\br{\frac{\tmix}{\mu^2 k}}\quad \forall k \in \{1,2,\ldots\}.
    \end{align*}
\end{theorem}
The proof of this theorem is presented in Appendix \ref{sec:expt_proof}. 
\section{Proof Outline}
\label{sec:proof_outline}
In this section, we outline the proofs for Theorems~\ref{thm:concentration_markov} and \ref{thm:expectation_markov}. The complete proofs are provided in Appendix~\ref{sec:thm_main} and Appendix~\ref{sec:expt_proof}, respectively.~In addition to standard arguments used in analyzing SGD under the P\L ~condition, the proof relies on solutions to the Poisson equation to handle the Markovian noise~\cite{metivier1984applications}. We also introduce a \textit{probabilistic induction} argument, which enables the application of concentration inequalities despite the absence of an almost-sure bound on the iterates.

The noisy gradient sample observed at every time $k$ can be decomposed as follows:
\begin{align*}
    G_k = \nabla f(x_k) + (g(x_k,Z_k) - \nabla f(x_k)) + M_{k+1}.
\end{align*}
Here $\nabla f(x_k)$ is the true noiseless gradient, and $M_{k+1}$ is the martingale difference noise. The term $(g(x_k,Z_k) - \nabla f(x_k))$ represents the Markovian noise. To handle this error term, we define the following solutions of the Poisson equation.
\begin{align}
    V(x,z) \coloneqq \bE\sqbr{\sum_{\ell=0}^{\infty}{(g(x,Z_\ell) - \nabla f(x))} \mid Z_0 = z}. \label{def:poisson_sol}
\end{align}
The following lemma states that $V(x,z)$ is a solution of the required Poisson equation.
\begin{lemma}\label{lemma:Poisson}
    For each $x\in\bR^d, z\in \cZ$, the function $V(x,z)$ satisfies the following.
    \begin{align}
        g(x,z) - \nabla f(x) = V(x,z) - \int{V(x,z\up) p(dz\up \mid z)}. \label{eq:poisson}
    \end{align}
\end{lemma}
The above lemma allows us to decompose the Markovian noise as follows:
\begin{align*}
    g(x_k,Z_k) - \nabla f(x_k)&=V(x_k,Z_k)-\int{V(x_k,z) p(dz \mid Z_k)}=\widetilde{M}_{k+1} - d_k,
\end{align*}
where,
\begin{align*}
    \widetilde{M}_{k+1} &\coloneqq V(x_k,Z_{k+1}) - \int{V(x_k,z) p(dz \mid Z_k)}, \\
    d_k &\coloneqq V(x_k,Z_{k+1}) - V(x_k,Z_k).
\end{align*}
Note that $\{\widetilde{M}_{k}\}$ is a martingale difference sequence with respect to the filtration $\{\cF_k\}$, i.e., $\bE[\widetilde{M}_{k+1}\mid\cF_k]=0$, $\forall k$. 
The term $\widetilde{M}_{k+1}$ can be combined with $M_{k+1}$ and handled together. Hence, the SGD iteration~\eqref{iter:sgd} can be rewritten as,
\begin{align*}
    x_{k+1} = x_k - \alpha_k (\nabla f(x_k) + M_{k+1} + \widetilde{M}_{k+1} - d_k), ~k = 0, 1, \ldots.
\end{align*}
Thus, the Poisson equation cleanly separates the Markovian noise into a martingale difference component and an additional correction term that can be shown to have a negligible effect on the rate. 

For ease of notation, we define
\begin{align}
    \zeta_{m,n} &\coloneqq \prod_{j=m}^{n}{\br{1 - \mu \alpha_j}}, \mbox{ for natural numbers } m\leq n.\label{def:zeta}
\end{align}
The following lemma presents an intermediate bound on $\Delta_k$.
\begin{lemma}\label{lem:iter_Delta}
    Suppose the setting of Theorem \ref{thm:concentration_markov} holds. Then for all $k\geq 0$,
    \begin{align}\label{eq:iter_Delta}
        \Delta_k
        &\leq \Delta_0 \zeta_{0,k-1} + \frac{CL}{2} \sum_{\ell=0}^{k-1}{\alpha_\ell^2 \zeta_{\ell+1,k-1}} - \sum_{\ell=0}^{k-1}{\alpha_\ell \zeta_{\ell+1,k-1} \left\langle \nabla f(x_\ell),\, M_{\ell+1} + \widetilde{M}_{\ell+1}\right\rangle} \notag\\
        &\quad
        +\sum_{\ell=0}^{k-1}\alpha_\ell \zeta_{\ell+1,k-1}
        \left\langle \nabla f(x_\ell),\, d_\ell \right\rangle .
    \end{align}
\end{lemma}

The term $\Delta_0 \zeta_{0,k-1} + \frac{CL}{2} \sum_{\ell=0}^{k-1}\alpha_\ell^2 \zeta_{\ell+1,k-1}$ can be bounded using standard properties of the stepsize sequence. The last term above can also be bounded using the properties of the solution of the Poisson equation~\eqref{def:poisson_sol} and using straightforward algebraic manipulations. It therefore remains to analyze the martingale term
$\sum_{\ell=0}^{k-1}\alpha_\ell \zeta_{\ell+1,k-1} \langle \nabla f(x_\ell), M_{\ell+1} + \widetilde{M}_{\ell+1} \rangle.$ To obtain a bound on the expected suboptimality gap, we take conditional expectations with respect to the sigma-fields $\cF_\ell$. Using the martingale difference property of $M_{\ell+1}$ and $\widetilde{M}_{\ell+1}$, each term in this summation has zero conditional expectation, which yields the result in Theorem~\ref{thm:expectation_markov}.

To obtain high-probability bounds, one would typically apply concentration inequalities such as the Azuma--Hoeffding inequality to this term. However, in our setting, we do not have a deterministic almost-sure bound on the martingale, which prevents a direct application of such inequalities. To address this issue, we use a \textit{probabilistic induction} approach.

We first explain the argument informally to build intuition. Suppose that up to time $\ell$ the iterates behave well, meaning that the suboptimality gap has been decaying at the desired rate; we denote this event by $E_\ell$. On this event, the martingale increments $M_{\ell+1}$ and $\widetilde{M}_{\ell+1}$ can be shown to be bounded. Consequently, the Azuma--Hoeffding inequality can be applied to control the stochastic term in the recursion, which implies that the desired bound on the suboptimality gap continues to hold at time $\ell+1$. Thus, if $E_\ell$ holds, the event $E_{\ell+1}$ also holds with high probability. The formal argument below implements this idea by defining the good events $E_k$ and controlling the associated martingale terms.

Define the following \textit{good} event for $k\geq 0$:
\begin{align} \label{def:Goodset}
    E_k \coloneqq \flbr{\Delta_k \leq \frac{K_0 \Delta_0 + \ind{k\neq 0} \br{\Gamma_1 + \Gamma_2 \log{\br{\frac{2 K_0}{\delta}}} + \Gamma_2 \log{\br{\frac{2k}{\delta}}}}}{k+K_0}}.
\end{align}
Here $\Gamma_1$ and $\Gamma_2$ are constants defined in \eqref{def:Gm1} and \eqref{def:Gm2}, respectively. We also define the sequence $\{\overline{M}_{\ell}\}$ as follows: 
\begin{equation}\label{Mbar}
    \overline{M}_{\ell+1} \coloneqq - \alpha_\ell \zeta_{\ell+1,k-1} \langle \nabla f(x_\ell), M_{\ell+1} + \widetilde{M}_{\ell+1}\rangle \ind{E_\ell}.
\end{equation}
Note that $\{\overline{M}_{\ell}\}$ is a martingale difference sequence with respect to the filtration $\{\cF_\ell\}$. The only difference in the definition of $\overline{M}_{\ell+1}$ and $-\alpha_\ell \zeta_{\ell+1,k-1}\langle \nabla f(x_\ell),\, M_{\ell+1} + \widetilde{M}_{\ell+1}\rangle$ appearing in Lemma \ref{lem:iter_Delta} is the multiplication with the indicator $\ind{E_\ell}$. On the event $E_{\ell}$, the suboptimality gap is bounded, which implies that function value and the gradient are bounded. Consequently, under the ABC condition, the noise terms are also bounded. Hence, $\overline{M}_{\ell+1}, l\geq 0$, has an almost sure bound, which allows us to apply the Azuma-Hoeffding inequality on the martingale $\sum_{\ell=0}^{k-1}\overline{M}_{\ell+1}, k\geq 1$.

It remains to relate the resulting high-probability bound on $\sum_{\ell=0}^{k-1}\overline{M}_{\ell+1}$ to the event that the suboptimality gap $\Delta_k$ decays at the desired rate. The constants $\Gamma_1$ and $\Gamma_2$ are chosen precisely so that this connection holds. Observe that our desired result is equivalent to obtaining an upper bound on $\bP\!\left(\cup_{k=0}^{\infty}\bar{E}_k\right)$, where $\bar{E}_k$ denotes the complement of $E_k$. The following lemma provides the key relation needed.
\begin{lemma}\label{lem:keylemma}
    If the setting of Theorem \ref{thm:concentration_markov} holds, then 
    \begin{align*}
        &\bP\br{\cup_{k=0}^{\infty}{\bar{E}_k}} \leq \sum_{k=1}^{\infty}{\bP\br{\flbr{\sum_{\ell=0}^{k-1}\overline{M}_{\ell+1} > \frac{\Gamma_2 \log{\br{\frac{2k}{\delta}}}}{k+K_0}}}}. 
    \end{align*}
\end{lemma}
The above lemma, together with an application of the Azuma--Hoeffding inequality, completes the proof of Theorem~\ref{thm:concentration_markov}.

The proof of Theorem~\ref{thm:concentration_mart} follows the same steps as that of Theorem~\ref{thm:concentration_markov}, except that the terms $\widetilde{M}_{\ell+1}$ and $d_\ell$ are absent in this case, which simplifies the analysis.
\section{Applications}\label{sec:applications}

In this section, we discuss three applications of the SGD algorithm when the objective function is $\mu$-P\L ~and the gradient is corrupted by Markov noise. We provide high-probability and in-expectation bounds on the values of the corresponding objective functions of the SGD iterates. The proofs of the propositions in this section are deferred to Appendix \ref{sec:appli_proofs}.

\subsection{Token Algorithm for Decentralized Linear Regression}\label{subsec:token_lr}
Consider a network consisting of $M$ agents (nodes). Each node $i$ for $i\in\{1,2,\ldots, M\}$ locally stores dataset $(\bA\uc{i},\bb\uc{i})$, where $\bA\uc{i} \in \bR^{N_i \times d}$ and $\bb\uc{i} \in \bR^{N_i}$ denote the design matrix and the corresponding target vector, respectively. Here, $N_i$ is the number of samples available at node $i$, and $d$ is the feature dimension. Let $\theta$ denote the parameter vector to be learned. The local objective function at the $i$th node is defined as
\begin{align*}
    \tilde{\cL}(\theta;i) := \frac{1}{2 N_i} \norm{\bA\uc{i} \theta - \bb\uc{i}}^2.
\end{align*}
The global objective is to minimize $$\cL(\theta) := \frac{1}{2N} \norm{\bA \theta - \bb}^2,$$ where $\bA \in \bR^{N \times d}$ and $\bb \in \bR^{N}$ are the global design matrix and target, respectively, obtained by stacking all $\bA\uc{i}$'s and $\bb\uc{i}$'s, and $N := N_1 + N_2 + \ldots + N_M$. By expanding the squared norm, one can show that $\cL(\theta) = \sum_{i=1}^{M}{q_i \tilde{\cL}(\theta;i)}$, where $q_i = N_i/N$. Note that $\cL$ is not strongly convex if $\bA\theta = \bb$ has infinitely many solutions, as is commonly encountered in fields where obtaining samples is expensive and features are high-dimensional, for example, statistical genomics~\cite{mei2016efficient}.

Decentralized algorithms are often preferred in settings where data is privacy-sensitive, such as individuals' health or financial records. A commonly used approach in such scenarios is the \textbf{token algorithm} \cite{hendrikx2023token}. In this framework, a token traverses the network of nodes according to a Markov chain $\{Z_k\}$ defined over the set of nodes, while carrying the current iterate. The movement is governed by a transition kernel $p$, which is designed in such a way that its unique invariant distribution is $\pi = (q_1, q_2, \ldots, q_M)$. At time $k$, the node visited by the token, $Z_k$, receives the global iterate $\theta_k$ and performs one local gradient descent update. The decentralized SGD iteration is as follows: 

\begin{align}\label{sgd:token_lr}
    \theta_{k+1} &= \theta_k - \alpha_k \nabla \tilde{\cL}(\theta_k;Z_k),
\end{align}
where $\alpha_k$ is the stepsize. We can rewrite the SGD update equation as follows:
\begin{align*}
    \theta_{k+1} &= \theta_k - \alpha_k \nabla \cL(\theta_k) + \alpha_k \br{\nabla \cL(\theta_k) - \nabla \tilde{\cL}(\theta_k;Z_k)},
\end{align*}
where $\nabla \cL(\theta_k) - \nabla \tilde{\cL}(\theta_k;Z_k) $ constitutes the Markov noise. 
\begin{remark}
    For simplicity, we have assumed that the gradient step at node $Z_k$ is computed using the entire local dataset. In practice, agents can also perform minibatch SGD, where a subset of datapoints is sampled uniformly and independently from $\{1,\ldots, N_{Z_k}\}$. In this case, the update includes an additional martingale difference noise term arising from the gradient estimate. This noise depends on the current iterate $\theta_k$. Our analysis can accommodate this setting, since the assumptions on the gradient noise allow it to scale with the gradient.
\end{remark}

The following proposition states some important properties of the loss function and the token algorithm.
\begin{proposition}\label{prop:token_lr}
    \begin{enumerate}[(a)]
        \item $\cL$ is $(\sigma_{\min}(\bA)/N)$-P\L, where $\sigma_{\min}(\bA)$ is the smallest non-zero singular value of $\bA$.
        \item $\cL$ is $(\sigma_{\max}^2(\bA)/N)$-smooth, where $\sigma_{\max}(\bA)$ is the largest singular value of $\bA$.
        \item $\tilde{\cL}(\cdot;i)$ is $(\sigma_{\max}^2(\bA\uc{i})/N_i)$-smooth for every $i$.
        \item For every $i$ and $\theta$, $\|\nabla \tilde{\cL}(\theta;i)\|^2 \leq 2 N \frac{\sigma_{\max}^2(\bA\uc{i})}{N_i^2} \cL(\theta)$.
    \end{enumerate}
\end{proposition}
Note that the problem setup satisfies all required assumptions stated in Section~\ref{sec:not&assum}, and thus Theorem~\ref{thm:concentration_markov} and Theorem~\ref{thm:expectation_markov} directly hold for this problem.

\begin{corollary}
    Consider the decentralized SGD algorithm~\eqref{sgd:token_lr} for linear regression. Let $\tmix$ denote the mixing time of the Markov chain $\{Z_k\}$ and let $\alpha_k = \frac{a}{k+K_0}$, where $a\geq 2N/\sigma_{\min(\bA)}$. Then there exists $\mathfrak{a}_1 = \Omega(\log(\tfrac{1}{\delta}))$ and a constant $\mathfrak{a}_2$, such that
    \begin{enumerate}[(a)]
        \item If $K_0 > \mathfrak{a}_1$, then with probability at least $1 - \delta$,
        \begin{align*}
            \cL(\theta_k) - \inf_{\theta}{\cL(\theta)} &\leq \cO\br{\frac{\tmix^2 N^2 \log{\br{\frac{k}{\delta}}}}{\sigma_{\min}(\bA)^2 k}},\quad \forall k \in \{1,2,\ldots\}.
        \end{align*}
        \item If $K_0 > \mathfrak{a}_2$, then
        \begin{align*}
            \bE\sqbr{\cL(\theta_k) - \inf_{\theta}{\cL(\theta)}} &\leq \cO\br{\frac{\tmix N^2}{\sigma_{\min}(\bA)^2 k}},\quad \forall k \in \{1,2,\ldots\}.
        \end{align*}
    \end{enumerate}
\end{corollary}

\subsection{Supervised Learning with Subsampling for Privacy Amplification}
We consider the supervised learning task where, given a dataset $D = \{\bA(i),\bb(i)\}_{i=1}^{N}$, $\bA(i) \in \bR^d, \bb(i) \in \bR$, and a parameterized class of functions $\phi(w;\bA(i))$, $w\in \bR^m$, the learner has to learn a parameter $w\ust$ such that $\phi(w\ust;\bA(i)) \approx \bb(i)$ for all $i \in \{1,2,\ldots, N\}$. $\bA \in \bR^{N\times d}$ denotes the collection of feature maps where the $i$-th row of $\bA$ is $\bA(i)^\top$, and $\bb$ denotes the target vector. We wish to minimize the squared loss,
\begin{align*}
    \cL(w) &:= \frac{1}{2N}\sum_{i=1}^{N}{(\phi(w;\bA(i)) - \bb(i))^2}.
\end{align*}

Differentially Private Stochastic Gradient Descent relies on randomized subsampling for privacy amplification~\cite{abadi2016deep}. Popular subsampling schemes with correlated noise include cyclic Poisson subsampling~\cite{choquette2023amplified}, in which the dataset is partitioned into subsets, and the datapoints are sampled from these subsets in a round-robin fashion, and balls-in-bins subsampling~\cite{choquette2025near}, which mimics shuffling by assigning each example to a random batch index within an epoch. Recently, \cite{dong2026privacy} introduced $b$-min-sep subsampling, a generalization of both Poisson and balls-in-bins subsampling.

In $b$-min-sep subsampling, the minibatch at time $k$, $B_k$, is constructed as follows. First, a subset $D_k$ of the dataset $D$ is created by excluding all samples that appeared in the previous $b-1$ minibatches, i.e., $D_k = D\setminus \cup_{i=1}^{b-1}{B_{k-i}}$. Then, each datapoint in $D_k$ is independently included in the current minibatch $B_k$ with a probability of $\rho$. This mechanism ensures that no example participates more than once in any window of $b$ steps, preserving the necessary structure for an efficient privacy analysis. 

This subsampling scheme can be equivalently described through a Markov chain governing the availability of each datapoint. Specifically, each datapoint occupies one of $b$ states labeled $0,1\ldots, b-1$, where state $0$ indicates that the datapoint is available for sampling in the next timestep. If a datapoint is selected in the minibatch, it transitions to state $b-1$. Thereafter, the state deterministically decreases by one at each step until it reaches $0$, where it again becomes eligible for sampling. Let $\zeta_k$ denote the $N$-dimensional random vector such that the $i$-th coordinate of $\zeta_k$, $\zeta_k(i)$, is the state of the $i$-th datapoint at time $k$. Define the loss function at time $k$,
$$
\tilde{\cL}(w;\zeta_k) = \frac{1}{2} \frac{\sum_{j=1}^{N}{\ind{\zeta_k(j) = b-1} (\phi(w;\bA(j)) - \bb(j))^2}}{\sum_{i=1}^{N}{\ind{\zeta_k(i)=b-1}}},
$$
with the convention that $\tilde{\cL}(w;\zeta_k)=0$ if $\ind{\zeta_k(i)=b-1} = 0$ for all $i$.~The $k$-th iteration of the SGD is then as follows:
\begin{align}\label{sgd:dp}
    w_{k+1} = w_k - \alpha_k \nabla \tilde{\cL}(w_k;\zeta_k),~k=0,1,\ldots.
\end{align}
Let $\pi$ denote the stationary distribution of the Markov chain $\{\zeta_k\}$. Although the exact form of the distribution $\pi$ is non-trivial, the following lemma shows that the stationary expectation of $\nabla \tilde{\cL}(w;\zeta_k)$ is the loss function we aim to minimize. 
\begin{lemma}\label{lem:markov_dpsgd}
    For all $w\in\bR^m$, $\cL(w) = \bE_{\zeta \sim \pi}\sqbr{\tilde{\cL}(w;\zeta)}$.
\end{lemma}
The SGD iteration can then be written as
\begin{align*}
    w_{k+1} = w_k - \alpha_k \nabla \cL(w_k) + \alpha_k (\nabla \cL(w_k) - \nabla \tilde{\cL}(w_k;\zeta_k)),~k=0,1,\ldots,
\end{align*}
where $\nabla \cL(w_k) - \nabla \tilde{\cL}(w_k;\zeta_k)$ constitutes the Markov noise. We need the following two technical assumptions on the model.

\begin{assumption}[Uniform conditioning]\label{assum:unif_cond}
    Let $\nabla\phi(w;\bA)$ be the differential of the map $\phi$ at $w$.~The tangent kernel of $\phi$ is defined as an $N \times N$ matrix $K(w) = \nabla \phi(w;\bA) \nabla \phi(w;\bA)^\top$. The smallest eigenvalue of the tangent kernel $K(w)$ is bounded below by $\mu$ for all $w$.
\end{assumption}

\begin{assumption}[Smoothness]\label{assum:smooth_dpsgd}
    The function $\tilde{\ell}_i(w) := \frac{1}{2} (\phi(w;\bA(j)) - \bb(j))^2$ is $L$-smooth for every $i \in \{1,2,\ldots,N\}$.
\end{assumption}

Assumption~\ref{assum:unif_cond} implies that $\cL(w)$ is $\mu$-P\L ~\cite{liu2022loss}. The same work showed that wide neural networks satisfy Assumption~\ref{assum:unif_cond} on a set of radius $R$, and $R$ is proportional to the width of the neural network. Assumption~\ref{assum:smooth_dpsgd} implies $L$-smoothness of $\cL(\cdot)$ and $\tilde{\cL}(\cdot,\zeta)$ for every $\zeta \in \{0,1,\ldots,b-1\}^N$. Moreover, Assumption~\ref{assum:smooth_dpsgd} implies the ABC condition. We show these results in the following proposition.

\begin{proposition}\label{prop:dp_sgd}
    \begin{enumerate}[(a)]
        \item Under Assumption~\ref{assum:unif_cond}, $\cL(\cdot)$ is $(\mu/N)$-P\L.
        \item Under Assumption~\ref{assum:smooth_dpsgd}, $\cL(\cdot;\zeta)$ is $L$-smooth for every $\zeta \in \{0,1, \ldots, b-1\}^N$. Further, $\cL(\cdot)$ is $L$-smooth.
        \item Under Assumption~\ref{assum:smooth_dpsgd}, for every $\zeta \in \{0,1, \ldots, b-1\}^N$,
        \begin{align*}
            \norm{\tilde{\cL}(w;\zeta)}^2 \leq \frac{2 L N}{\sum_{i=1}^{N}{\ind{\zeta(i)=b-1}}} \cL(w).
        \end{align*}
    \end{enumerate}
\end{proposition}
Having shown that the loss function and the noisy gradient samples satisfy all assumptions from Section \ref{sec:not&assum}, we can now state the following convergence results.
\begin{corollary}
    Consider the SGD algorithm~\eqref{sgd:dp} for the differentially private supervised learning setup described above that satisfies Assumptions~\ref{assum:unif_cond} and \ref{assum:smooth_dpsgd}. Let $\tmix$ denote the mixing time of the Markov chain $\{\zeta_k\}$ and let $\alpha_k = \frac{a}{k+K_0}$, where $a\geq 2N/\mu$. Then there exists $\mathfrak{a}_3 = \Omega(\log(\tfrac{1}{\delta}))$ and a constant $\mathfrak{a}_4$, such that
    \begin{enumerate}[(a)]
        \item If $K_0 > \mathfrak{a}_3$, then with probability at least $1 - \delta$,
        \begin{align*}
            \cL(w_k) - \inf_{w}{\cL(w)} &\leq \cO\br{\frac{\tmix^2 N^2 \log{\br{\frac{k}{\delta}}}}{\mu^2 k}},\quad \forall k \in \{1,2,\ldots\}.
        \end{align*}
        \item If $K_0 > \mathfrak{a}_4$, then
        \begin{align*}
            \bE\sqbr{\cL(w_k) - \inf_{w}{\cL(w)}} &\leq \cO\br{\frac{\tmix N^2}{\mu^2 k}},\quad \forall k \in \{1,2,\ldots\}.
        \end{align*}
    \end{enumerate}
\end{corollary}

\subsection{Online System Identification} Let the process $\{Z_k\}$ evolve as $Z_{k+1} = A\lst Z_k + w_k$, where $A\lst \in \bR^{d \times d}$, and $w_k,k\geq 0$, is i.i.d. noise with zero mean and bounded variance. The objective of the online system identification problem is to estimate $A\lst$ ~\cite{even23a,kowshik2021streaming}. Consider that an optimizer is trying to estimate the system parameter $A\lst$ by minimizing the squared distance loss,
\begin{align*}
    \cL(A) := \frac{1}{2} \int{\norm{Az - A\lst z}^2 d\pi(z)}.
\end{align*}

Here, $\pi$ denotes the stationary distribution of the Markov process $\{Z_k\}$. The following assumption guarantees its existence and uniqueness.
\begin{assumption}\label{assum:stable}
    The eigenvalues of $A\lst$ are strictly inside the unit disc, i.e., $\lambda_{\max} := \max_{i=1,2,\ldots,d}{\abs{\lambda_i}} < 1$, where $\lambda_i$'s are the eigenvalues of $A\lst$.
\end{assumption}
\begin{proposition}\label{prop:stationary_dist}
    Let $A\lst$ satisfy Assumption~\ref{assum:stable}. Then there exists a unique stationary distribution $\pi$ for the Markov chain $\{Z_k\}$. Moreover, for any initial condition $Z_0$, the distribution of $Z_k$ converges to $\pi$ at a geometric rate.
\end{proposition}
The loss observed at time $k$ is denoted by $\tilde{\cL}(A,Z_k)$, where
\begin{align*}
    \tilde{\cL}(A,Z_k) := \frac{1}{2} \norm{A Z_k - A\lst Z_k}^2.
\end{align*}
We consider the SGD update,
\begin{align}\label{sgd:sysid}
    A_{k+1} &= A_k - \alpha_k (A_k Z_k - Z_{k+1}) Z_k^\top.
\end{align}
Replacing $Z_{k+1}$ with $A\ust Z_k +w_k$, we rewrite the update as follows:
\begin{align*}
    A_{k+1} &= A_k - \alpha_k (A_k Z_k - A\lst Z_k) Z_k^\top + \alpha_k w_k Z_k^\top \\
    &= A_k - \alpha_k \nabla \cL(A_k) + \alpha_k \br{\nabla \cL(A_k) - \nabla \tilde{\cL}(A_k,Z_k)} + \alpha_k w_k Z_k^\top,
\end{align*}
where $\nabla \cL(A_k) - \nabla \tilde{\cL}(A_k,Z_k)$ is the Markovian noise and $w_k Z_k^\top$ is the martingale difference noise component. The following assumption is required for the Markov chain to lie in a compact state space, and for the P\L~condition to hold. 

\begin{assumption}\label{assum:sysid}
    (i) There exists a constant $B>0$ such that $\norm{w_k} \leq B$ almost surely for every $k$.
    (ii) $\Sigma := \int{z z^\top d\pi(z)}$ is strictly positive definite with minimum eigenvalue $\mu_{\min} > 0$.
\end{assumption}

\begin{proposition}\label{prop:sysid}
    Under Assumption~\ref{assum:stable} and Assumption~\ref{assum:sysid},
    \begin{enumerate}[(a)]
        \item $\norm{Z_k} \leq \max{\{\norm{Z_0}, B/(1-\lambda_{\max})\}}$ almost surely for every $k$.
        \item $\cL(\cdot)$ is $\mu_{\min}$-P\L.
        \item $\cL(\cdot)$ is $\mu_{\max}$-smooth, where $\mu_{\max}$ is the largest eigenvalue of $\Sigma$.
        \item $\norm{\nabla \tilde{\cL}(A_k,Z_k)}_F^2 \leq \frac{2 B^4}{(1 - \lambda_{\max})^4 \mu_{\min}} \cL(A)$ almost surely for every $k$.
        \item $\norm{w_k Z_k^\top}^2 \leq \frac{B^4}{(1 - \lambda_{\max})^2}$ almost surely for every $k$.
    \end{enumerate}
\end{proposition}
Propositions \ref{prop:stationary_dist} and \ref{prop:sysid} together show that all assumptions in Section \ref{sec:not&assum} are satisfied, and hence we have the following bound on the rate of convergence.
\begin{corollary}
    Consider the SGD algorithm~\eqref{sgd:sysid} for the system identification problem described above that satisfies Assumption~\ref{assum:sysid}. Let $\tmix$ denote the mixing time of the Markov chain $\{Z_k\}$ and let $\alpha_k = \frac{a}{k+K_0}$, where $a\geq 2/\mu_{\min}$. Then there exists $\mathfrak{a}_5 = \Omega(\log(\tfrac{1}{\delta}))$ and a constant $\mathfrak{a}_6$, such that
    \begin{enumerate}[(a)]
        \item If $K_0 > \mathfrak{a}_5$, then with probability at least $1 - \delta$,
        \begin{align*}
            \cL(A_k) &\leq \cO\br{\frac{\tmix^2 \log{\br{\frac{k}{\delta}}}}{\mu_{\min}^2 k}},\quad \forall k \in \{1,2,\ldots\}.
        \end{align*}
        \item If $K_0 > \mathfrak{a}_6$, then
        \begin{align*}
            \bE\sqbr{\cL(A_k)} &\leq \cO\br{\frac{\tmix}{\mu_{\min}^2 k}},\quad \forall k \in \{1,2,\ldots\}.
        \end{align*}
    \end{enumerate}
\end{corollary} 
\section{Conclusion}\label{sec:conclusion}

In this work, we establish the first uniform-in-time concentration bound on the suboptimality of SGD iterates under the P\L~condition with Markovian noise. Our analysis relies on the Poisson equation to handle the Markov noise, and a novel \textit{probabilistic induction} argument to deal with the lack of almost sure bounds on the loss function. We also demonstrate the practical relevance of our framework by analyzing three representative applications. Future directions include extending these results to broader classes of objective functions, such as those satisfying the Kurdyka-\L ojasiewicz condition and establishing bounds for SGD with asynchronous updates, which would enable a more realistic treatment of decentralized optimization.

\appendix
\section{Proofs required for Theorem~\ref{thm:concentration_markov}}\label{sec:thm_main}

\subsection{Proof for Lemma \ref{lemma:Poisson}}
\begin{proof}
    For a natural number $N$, define $V_N(x,z)$ as follows:
    \begin{align*}
        V_N(x,z) &\coloneqq \bE\sqbr{\sum_{\ell=0}^{N}{(g(x,Z_\ell) -\nabla f(x))} \mid Z_0 = z} \\
        &= \sum_{\ell=0}^{N}{\br{\int{g(x,z\up) p\uc{\ell}(dz\up \mid z)} - \int{g(x,z\up) \pi(d z\up)}}}.
    \end{align*}
    Note that for any $N$,
    \begin{align*}
        \norm{V_N(x,z)} &= \norm{\sum_{k=0}^{N}{\int{g(x,z\up) \br{p\uc{k}(z\up \mid z) - \pi(z\up)} dz\up}}} \notag\\
        &\stackrel{(a)}{\leq} \sqrt{d} \sup_{z\up}{\norm{g(x,z\up)}} \sum_{k=0}^{\infty}{\norm{p\uc{k}(\cdot \mid z) - \pi}_{TV}} \notag\\
        &\stackrel{(b)}{\leq} \sqrt{d} \sup_{z\up}{\norm{g(x,z\up)}} \sum_{k=0}^{\infty}{2^{-\floor{k/\tmix}}} \notag\\
        &= \sqrt{d} \sup_{z\up}{\norm{g(x,z\up)}} \sum_{k=0}^{\infty}{\sum_{i=0}^{\tmix-1}{2^{-\floor{(k\cdot\tmix+i)/\tmix}}}} \notag\\
        &\leq \sqrt{d} \sup_{z\up}{\norm{g(x,z\up)}} \tmix \sum_{k=0}^{\infty}{2^{-k}} \notag\\
        &\stackrel{(c)}{\leq} 2 \tmix \sqrt{d} \sqrt{A \norm{\nabla f(x)}^ 2 + B (f(x) - f\ust) + C},
    \end{align*}
    where $(a)$ follows from Lemma~\ref{lem:bdd_dotdifLv}, $(b)$ follows from the definition of $\tmix$~\eqref{def:tmix}, and $(c)$ follows from Assumption~\ref{assum:lipg}.
    
    Using the same set of steps as above, for two natural numbers $N_1 \leq N_2$,
    \allowdisplaybreaks
    \begin{align*}
        \norm{V_{N_1}(x,z) - V_{N_2}(x,z)} &= \norm{\sum_{\ell=N_1 + 1}^{N_2}{\br{\int{g(x,z\up) p\uc{\ell}(dz\up \mid z)} - \int{g(x,z\up) \pi(d z\up)}}}} \\
        &\leq \sqrt{d} \sup_{z\up \in \cZ}{\norm{g(x,z\up)}} \sum_{\ell=N_1 + 1}^{\infty}{\norm{p\uc{\ell}(\cdot \mid z) - \pi(\cdot)}}_{TV}\\
        &\leq 2 \tmix \sqrt{d} \sup_{z\up \in \cZ}{\norm{g(x,z\up)}} 2^{-\floor{N_1/\tmix}} \\
        &\leq \br{2 \tmix \sqrt{d} \sqrt{A \norm{\nabla f(x)}^ 2 + B (f(x) - f\ust) + C}} 2^{-\floor{N_1/\tmix}}.
    \end{align*}
     Thus, $\{V_N(x,\cdot)\}$ is bounded, Cauchy sequence for every $x$. Hence, $\lim_{N \to \infty}{V_N(x,\cdot)}$ exists and $\lim_{N \to \infty}{V_N(x,\cdot)} = V(x,\cdot)$. 
    
    Integrating $V_N$ with respect to the distribution $p(\cdot \mid z)$, we obtain
    \begin{align*}
        \int{V_N(x,z\up) p(dz\up \mid z)} &= \sum_{\ell=0}^{N}{\int{\br{\int{g(x,z\upp) p\uc{\ell}(dz\upp \mid z\up)} - \nabla f(x)} p(dz\up \mid z)}} \\
        &= \sum_{\ell=1}^{N+1}{\int{g(x,z\up) p\uc{\ell}(dz\up \mid z)}} - \sum_{\ell=0}^{N}{\nabla f(x)}.
    \end{align*}
    Subtracting $\int{V_N(x,z\up) p(dz\up \mid z)}$ from $V_N(x,z)$ we get,
    \begin{align*}
        V_N(x,z) - \int{V_N(x,z\up) p(dz\up \mid z)} = g(x,z) - \int{g(x,z\up) p\uc{N+1}(dz\up \mid z)}.
    \end{align*}
     Taking the limit $N \to \infty$ above, and using the dominated convergence theorem~\cite[p. 54]{folland2013real} and the ergodicity property of the Markov chain, we have
    $$V(x,z) - \int{V(x,z\up) p(dz\up \mid z)} = g(x,z) - \nabla f(x).$$
    This concludes the proof for Lemma \ref{lemma:Poisson}.
\end{proof}

\subsection{Proof for Lemma \ref{lem:iter_Delta}}
\begin{proof}
    For convenience, denote $W_{\ell+1} := M_{\ell+1} + \widetilde{M}_{\ell+1} - d_\ell$. First, we obtain a recursive relation by bounding $f(x_{k+1})$ in terms of $f(x_k)$.
    \begin{align}\label{lem:iter_Delta;ineq:1}
        f(x_{k+1}) &\stackrel{(a)}{\leq} f(x_k) - \alpha_k \angl{\nabla f(x_k), G_k} + \frac{L}{2} \alpha_k^2 \norm{G_k}^2 \notag\\
        &\stackrel{(b)}{\leq} f(x_k) - \alpha_k \br{\norm{\nabla f(x_k)}^2 + \angl{\nabla f(x_k), W_{k+1}}} \notag\\
        &\quad + \frac{L}{2} \alpha_k^2 \br{A \norm{\nabla f(x_k)}^2 + B \br{f(x_k) - f\ust} + C} \notag\\
        &= f(x_k) - \br{\alpha_k - \frac{A L}{2} \alpha_k^2} \norm{\nabla f(x_k)}^2 + \frac{B L}{2} \alpha_k^2 \br{f(x_k) - f\ust} \notag \\
        &\quad  + \frac{C L}{2} \alpha_k^2 - \alpha_k \angl{\nabla f(x_k), W_{k+1}} \notag\\
        &\stackrel{(c)}{\leq} f(x_k) - \br{2 \mu \alpha_k - \br{2 \mu A+ B} \frac{L}{2} \alpha^2_k} \br{f(x_k) - f\ust} + \frac{C L}{2} \alpha_k^2 \notag \\
        &\quad - \alpha_k \angl{\nabla f(x_k), W_{k+1}} 
    \end{align}
    Here, inequality $(a)$ follows from smoothness of $f$ (See \cite[Lemma~3.4]{bubeck2015convex}) and $(b)$ follows from the definition of $G_k$ and from Assumption~\ref{assum:noisy_grad_norm}. Inequality (c) is obtained using the P\L~condition~\eqref{eq:pl_cond} along with our choice of $K_0$ which ensures that $\alpha_k - A L \alpha_k^2 \geq 0$ for all $k\geq 0$. Subtracting $f\ust$ from both sides of above equation, we get that,
    \begin{align*}
        \Delta_{k+1} &\leq \br{1 - 2 \mu \alpha_k + \br{2 \mu A+ B} \frac{L}{2} \alpha^2_k} \Delta_k + \frac{C L}{2} \alpha_k^2 - \alpha_k \angl{\nabla f(x_k), W_{k+1}}.
    \end{align*}
    Further, letting $K_0 \geq \frac{a L}{2} \br{2 A + \frac{B}{\mu}}$, we have $1 - 2 \mu \alpha_k + \br{2 A \mu + B} \frac{L}{2} \alpha^2_k \leq 1 - \mu \alpha_k$, for all $k \in \{0,1,\ldots\}$. Thus,
    \begin{align*}
        \Delta_{k+1} &\leq \br{1 - \mu \alpha_k} \Delta_k + \frac{C L}{2} \alpha_k^2 - \alpha_k \angl{\nabla f(x_k), W_{k+1}}.
    \end{align*}
    Writing $\Delta_\ell$ in terms of $\Delta_{\ell-1}$ recursively for $\ell = k$ to $1$ and expanding $W_{\ell+1}$, we obtain
    \begin{align}\label{ub:Delta_k}
        \Delta_k &\leq \underbrace{\Delta_0 \zeta_{0,k-1} + \frac{C L}{2} \sum_{\ell=0}^{k-1}{\alpha_\ell^2 \zeta_{\ell+1,k-1}}}_{\cT_1} + \underbrace{\sum_{\ell=0}^{k-1}{\alpha_\ell \zeta_{\ell+1,k-1} \langle \nabla f(x_\ell), d_\ell\rangle}}_{\cT_2} \\
        &\quad - \sum_{\ell=0}^{k-1}{\alpha_\ell \zeta_{\ell+1,k-1} \langle \nabla f(x_\ell), M_{\ell+1} + \widetilde{M}_{\ell+1}\rangle}. \notag
    \end{align}
\end{proof}

\subsection{Proof for Lemma \ref{lem:keylemma}}
\begin{proof}
    For simplicity, denote
    \begin{align}\label{def:Lambda}
        \Lambda(k,\delta) &\coloneqq K_0 \Delta_0 + \Gamma_1 + \Gamma_2 \log{\br{\frac{2 K_0}{\delta}}} + \Gamma_2 \log{\br{\frac{2k}{\delta}}},
    \end{align}
    where $\Gamma_1$ and $\Gamma_2$ are defined in \eqref{def:Gm1} and \eqref{def:Gm2}, respectivly.~We will use $\Lambda(k,\delta)$ throughout the proof. Note that,
    \begin{align*}
        \cup_{k=0}^{\infty}{\bar{E}_k} &= \cup_{k=0}^{\infty}{\br{\cap_{\ell=0}^{k-1}{E_\ell}}\cap \bar{E}_k}.
    \end{align*}
    To see this, if $\omega \in \cup_{k=0}^{\infty}{\bar{E}_k}$, then there exists a $k_0$ such that $\omega \in \bar{E}_{k_0}$ but $\omega \notin \bar{E}_{k}$ for all $k < k_0$, that is, $\omega \in (\cap_{\ell=0}^{k_0-1}{E_\ell}) \cup \bar{E}_{k_0}$. Hence,
    \begin{align*}
        \omega \in \cup_{k=0}^{\infty}{\br{\cap_{\ell=0}^{k-1}{E_\ell}}\cap \bar{E}_k}.
    \end{align*}
    The converse is trivial. Hence, using the union bound, we can write,
    \begin{align}
        \bP\br{\cup_{k=0}^{\infty}{\bar{E}_k}} &\leq \bP(\bar{E}_0) + \sum_{k=1}^{\infty}{\bP\br{\br{\cap_{\ell=0}^{k-1}{E_\ell}}\cap \bar{E}_k}} \notag \\
        &= \sum_{k=1}^{\infty}{\bP\br{\br{\cap_{\ell=0}^{k-1}{E_\ell}}\cap \bar{E}_k}}. \label{badset:1}
    \end{align}
    Here $\bP(\bar{E}_0)= 0$ by definition as $\Delta_0 < \frac{\Lambda(0,\delta)}{K_0}$. Then for $k \in \{1,2,\ldots\}$,
    \begin{align}
        \bar{E}_k &= \flbr{\Delta_k > \frac{\Lambda(k,\delta)}{k+K_0}} \notag \\
        &\stackrel{(a)}{\subseteq} \left\{\Delta_0 \zeta_{0,k-1} + \frac{C L}{2} \sum_{\ell=0}^{k-1}{\alpha_\ell^2 \zeta_{\ell+1,k-1}} - \sum_{\ell=0}^{k-1}{\alpha_\ell \zeta_{\ell+1,k-1} \langle \nabla f(x_\ell), M_{\ell+1} + \widetilde{M}_{\ell+1}\rangle} \right. \notag\\ 
        &\qquad \left. + \sum_{\ell=0}^{k-1}{\alpha_\ell \zeta_{\ell+1,k-1} \langle \nabla f(x_\ell), d_\ell\rangle} > \frac{\Lambda(k,\delta)}{k+K_0}\right\}, \label{badset:2}
    \end{align}
    where $(a)$ follows from Lemma~\ref{lem:iter_Delta}. From Lemma~\ref{lem:bdd_determ_terms}, we have that 
    \begin{align*}
        \Delta_0 \zeta_{0,k-1} + \frac{C L}{2} \sum_{\ell=0}^{k-1}{\alpha_\ell^2 \zeta_{\ell+1,k-1}} \leq \frac{K_0 \Delta_0 + e a^2 C L/2}{k+K_0}.
    \end{align*}
    It follows from Lemma~\ref{lem:Delta_const_bd} that on the event $E_\ell$, 
    \begin{align*}
        \Delta_\ell \leq \frac{K_0 \Delta_0 + \Gamma_1 + 2 \Gamma_2 \log{\br{\frac{2 K_0}{\delta}}}}{K_0} .
    \end{align*}
    Using Lemma~\ref{lem:bdd_markov_residue} and the above bound, we obtain that, on the event $\cap_{\ell = 1}^{k-1}{E_\ell}$,
    \begin{align*}
        \sum_{\ell=0}^{k-1}{\alpha_\ell \zeta_{\ell+1,k-1} \angl{\nabla f(x_\ell), d_\ell}} &\leq \frac{D_1 + D_2 \frac{K_0 \Delta_0 + \Gamma_1 + 2 \Gamma_2 \log{\br{\frac{2 K_0}{\delta}}}}{K_0}}{k+K_0} \\
        &\stackrel{(a)}{\leq} \frac{D_1 + D_2 \Delta_0 + (\Gamma_1/2) + \Gamma_2 \log{\br{\frac{2 K_0}{\delta}}}}{k+K_0},
    \end{align*}
    where $D_1$ and $D_2$ are defined in \eqref{def:D_1} and in \eqref{def:D_2}, respectively. The inequality $(a)$ is obtained by letting $K_0 \geq 2 D_2$. Thus, on the event $\cap_{\ell = 1}^{k-1}{E_\ell}$,
    \begin{align}
        &\Delta_0 \zeta_{0,k-1} + \frac{C L}{2} \sum_{\ell=0}^{k-1}{\alpha_\ell^2 \zeta_{\ell+1,k-1}} + \sum_{\ell=0}^{k-1}{\alpha_\ell \zeta_{\ell+1,k-1} \angl{\nabla f(x_\ell), d_\ell}} \notag \\
        &\leq \frac{K_0 \Delta_0 + \Gamma_1 + \Gamma_2 \log{\br{\frac{2 K_0}{\delta}}}}{k+K_0}. \label{bdd:non_mart_terms}
    \end{align}
    Note that $\Gamma_1 = ea^2 C L + 2 \br{D_1 + D_2 \Delta_0}$ which ensures that $ea^2CL/2+D_1+D_2\Delta_0+\Gamma_1/2=\Gamma_1$.
    Invoking \eqref{bdd:non_mart_terms} and \eqref{badset:2} in \eqref{badset:1}, we have that,
    \begin{align*}
        &\bP\br{\cup_{k=0}^{\infty}{\bar{E}_k}} = \sum_{k=1}^{\infty}{\bP\br{\br{\cap_{\ell=0}^{k-1}{E_\ell}}\cap \bar{E}_k}}\\
        &\leq \sum_{k=1}^{\infty}{\bP\br{\br{\cap_{\ell=0}^{k-1}{E_\ell}} \cap \flbr{- \sum_{\ell=0}^{k-1}{\alpha_\ell \zeta_{\ell+1,k-1} \langle \nabla f(x_\ell), M_{\ell+1} + \widetilde{M}_{\ell+1}\rangle} > \frac{\Gamma_2 \log{\br{\frac{2k}{\delta}}}}{k+K_0}}}} \\
        &= \sum_{k=1}^{\infty}{\bP\br{\br{\cap_{\ell=0}^{k-1}{E_\ell}} \cap \flbr{\sum_{\ell=0}^{k-1}\overline{M}_{\ell+1} > \frac{\Gamma_2 \log{\br{\frac{2k}{\delta}}}}{k+K_0}}}} \\
        &\leq \sum_{k=1}^{\infty}{\bP\br{\flbr{\sum_{\ell=0}^{k-1}\overline{M}_{\ell+1} > \frac{\Gamma_2 \log{\br{\frac{2k}{\delta}}}}{k+K_0}}}}.
    \end{align*}
    Here $\overline{M}_{\ell+1}$ is as defined in \eqref{Mbar}:
    $$\overline{M}_{\ell+1}=\alpha_\ell \zeta_{\ell+1,k-1} \langle \nabla f(x_\ell), M_{\ell+1} + \widetilde{M}_{\ell+1}\rangle\ind{E_{\ell}}.$$
    This concludes the proof.
\end{proof}

\subsection{Proof of Theorem~\ref{thm:concentration_markov}}
\begin{theorem}[Detailed Version of Theorem~\ref{thm:concentration_markov}]\label{thm:concentration_markov_full}
    Denote
    \begin{align*}
        \nu_1 &\coloneqq 32 (ae)^2 L (\tmix^2 d + 1) \br{\frac{2 A L + B}{2 \mu a - 3} \br{2\Delta_0 + \frac{D_1}{D_2} + \frac{ea^2 CL}{D_2}} + \frac{C}{2 \mu a - 2}}, \\
        \nu_2 &\coloneqq 64 (ae)^2 L (\tmix^2 d + 1) \frac{2 A L + B}{2 \mu a - 3},
    \end{align*}
    where $D_1$ and $D_2$ are as defined in \eqref{def:D_1} and \eqref{def:D_2}, respectively. Let the objective function $f:\bR^d \to \bR$ be $\mu$-P\L~\eqref{eq:pl_cond} and let 
    \begin{align}\label{lb:K_0}
        K_0 \geq \max&\left\{\frac{a L}{2} \br{2 A + \frac{B}{\mu}}, \mu a, 2 D_2,\right.\\
        &\left. ~24 \nu_2 (2 \log(48 \nu_2) + 1) \log\br{\frac{48 \nu_2 (2 \log(48 \nu_2) + 1)}{\delta}} \right\}. \notag
    \end{align}
    Under Assumption~\ref{assum:smoothness}--\ref{assum:lipg}, the sub-optimality of SGD iterates are bounded on a set of probability at least $1 - \delta$, where $\delta \in (0,1)$, as follows: for every $k \in \{1,2,\ldots\}$,
    \begin{align*}
        \Delta_k &\leq \frac{K_0 \Delta_0 + \Gamma_1 + \Gamma_2 \log{\br{\frac{2 K_0}{\delta}}} + \Gamma_2 \log{\br{\frac{2k}{\delta}}}}{k + K_0},
    \end{align*}
    where
    \begin{align}
        \Gamma_1 &\coloneqq ea^2 C L + 2 \br{D_1 + D_2 \Delta_0}, \label{def:Gm1} \\
        \Gamma_2 &\coloneqq 4 \nu_1 \br{1 + 3\; \overline{\log\br{K_0}}} + 2 \sqrt{\nu_1 \br{\overline{K}_0 \Delta_0 + 2 \Gamma_1}}. \label{def:Gm2}
    \end{align}
    Here $\overline{K}_0$ and $\overline{\log\br{K_0}}$ are constants independent of $\delta$ such that $\overline{K}_0 \geq K_0 / \log(2/\delta)$ and $\overline{\log\br{K_0}} \geq \log\br{\frac{2K_0}{\delta}} /\log\br{\frac{2}{\delta}}$.
\end{theorem}

\begin{proof}
    From Lemma~\ref{lem:Delta_const_bd}, we have that
    \begin{align*}
        \frac{\Lambda(k,\delta)}{k+K_0} \leq \br{\Delta_0 + \frac{\Gamma_1}{K_0}} + \frac{2 \Gamma_2}{K_0} \log{\br{\frac{2K_0}{\delta}}}.
    \end{align*}
    Invoking this in Lemma~\ref{lem:square_sum_bdd}, we have
    \begin{align*}
        2\sum_{\ell=0}^{k-1}{\overline{M}_{\ell+1}^2} \leq \frac{\br{\nu_1 + \frac{\nu_2 \Gamma_2}{K_0} \log\br{\frac{2K_0}{\delta}}} \Lambda(k,\delta)}{\br{k+K_0}^2}.
    \end{align*}
    Using the above in Azuma-Hoeffding inequality (Lemma~\ref{lem:ah_ineq}), we can write that for any $k \in \{1,2,\ldots\}$,
    \begin{align*}
        &\bP\br{\flbr{-\sum_{\ell=0}^{k-1}\overline{M}_{\ell+1} > \eps}} \leq \exp\br{\frac{-\eps^2 (k+K_0)^2}{\br{\nu_1 + \frac{\nu_2 \Gamma_2}{K_0} \log\br{\frac{2K_0}{\delta}}} \Lambda(k,\delta)}},
    \end{align*}
    or
    \begin{align*}
        \bP\left(\left\{\sum_{\ell=0}^{k-1}\overline{M}_{\ell+1} \!>\! \frac{1}{k + K_0} \sqrt{2 \br{\nu_1 + \frac{\nu_2 \Gamma_2}{K_0} \log\br{\frac{2K_0}{\delta}}} \Lambda(k,\delta) \log{\br{\frac{2k}{\delta}}}} \right\}\right) \leq \frac{6}{\pi^2} \frac{\delta}{k^2},
    \end{align*}
    where we have used the fact that $2\log(\frac{2k}{\delta})\geq \log(\frac{k^2}{\delta^2}\frac{\pi^2}{6})$. Based on Lemma \ref{lemma:Gamma2}, we have chosen $\Gamma_2$ such that 
    $$\Gamma_2\log\left(\frac{2k}{\delta}\right)\geq \sqrt{2 \br{\nu_1 + \frac{\nu_2 \Gamma_2}{K_0} \log\br{\frac{2K_0}{\delta}}} \Lambda(k,\delta) \log{\br{\frac{2k}{\delta}}}}.$$
    This implies that 
    \begin{align*}
        \bP\br{\flbr{\sum_{\ell=0}^{k-1}\overline{M}_{\ell+1} > \frac{\Gamma_2 \log{\br{\frac{2k}{\delta}}}}{k + K_0}}} &\leq \frac{6}{\pi^2} \frac{\delta}{k^2}.
    \end{align*}
    Using a union bound for $k = 1, 2, \ldots$, we have,
    \begin{align*}
        &\bP\left(\left\{\sum_{\ell=0}^{k-1}\overline{M}_{\ell+1}\geq \frac{\Gamma_2 \log{\br{\frac{2k}{\delta}}}}{k + K_0} \mbox{ for every } k \in \bN \right\}\right) \leq \delta.
    \end{align*}
    The proof is then complete by Lemma~\ref{lem:keylemma}.
\end{proof}

\subsection{Additional Lemmas}
\subsubsection{Bounding \texorpdfstring{$\cT_1$}{T1}}
We derive an upper bound on $\cT_1$ ~\eqref{ub:Delta_k} in the following lemma.
\begin{lemma}\label{lem:bdd_determ_terms}
    Let the stepsize sequence $\{\alpha_k\}$ be of the same form as in Lemma~\ref{lem:iter_Delta}, then
    \begin{align*}
        \Delta_0 \zeta_{0,k-1} + \frac{C L}{2} \sum_{\ell=0}^{k-1}\alpha_\ell^2 \zeta_{\ell+1,k-1} \leq \frac{K_0 \Delta_0 + e a^2 C L/2}{k+K_0}, ~\forall k \in \{0,1,\ldots\}.
    \end{align*}
\end{lemma}
\begin{proof}
    Applying Lemma~\ref{lem:zeta_bdd} with $m=0,n=k-1$, and $K_0 \geq \mu a$,
    \begin{align}
        \Delta_0 \zeta_{0,k-1} \le \frac{K_0 \Delta_0}{k+K_0}. \label{bdd:determ_a}
    \end{align}
    From Lemma~\ref{lem:bddalphazeta} and from the assumption that $K_0 \geq \mu a > 1$, we have that,
    \begin{align}
        \sum_{\ell=0}^{k-1}\alpha_\ell^2 \zeta_{\ell+1,k-1} \leq \frac{e a^2}{k+K_0}.\label{bdd:determ_b}
    \end{align}
    Combining \eqref{bdd:determ_a} and \eqref{bdd:determ_b}, the proof is complete.
\end{proof}

\subsubsection{Bounding \texorpdfstring{$\cT_2$}{T2}}
We derive an upper bound on $\cT_2$~\eqref{ub:Delta_k}, the residual term after converting the Markov noise to martingale difference noise, in the following lemma.
\begin{lemma} \label{lem:bdd_markov_residue}
    Let the stepsize sequence $\{\alpha_k\}$ be of the same form as in Lemma~\ref{lem:iter_Delta}.
    
    (a) If $\max_{\ell < k}{\Delta_\ell} \leq \Delta_{\max}$, then
    \begin{align*}
        \sum_{\ell=0}^{k-1}{\alpha_\ell \zeta_{\ell+1,k-1} \angl{\nabla f(x_\ell), d_\ell}} \leq  \frac{D_1 + D_2 \Delta_{\max}}{k+K_0},
    \end{align*}
    where
    \begin{align}
        D_1 &\coloneqq 2 a \mfr_1 \tmix L \sqrt{d} \Delta_0 + 10 a \mfr_2 \tmix \sqrt{d} + \frac{e a^2 \mfr_4 \tmix (L + L_g) \sqrt{d}}{\mu a - 1}, \mbox{ and}\label{def:D_1} \\
        D_2 &\coloneqq 8 a \mfr_1 \tmix L \sqrt{d} + \frac{e a^2 \mfr_3 \tmix (L + L_g) L \sqrt{d}}{\mu a - 1}. \label{def:D_2}
    \end{align}
    $\mfr_1, \mfr_2, \mfr_3$ and $\mfr_4$ are constants that depend on $A, B$ and $C$.
    
    (b) On the event $\cap_{\ell=1}^{k-1}{E_\ell}$,
    \begin{align*}
        \sum_{\ell=0}^{k-1}{\alpha_\ell \zeta_{\ell+1,k-1} \angl{\nabla f(x_\ell), d_\ell}} \leq \frac{\br{D_1 + D_2 \Delta_0 + \frac{D_2 \Gamma_1}{K_0}} + \frac{2 D_2 \Gamma_2}{K_0} \log{\br{\frac{2 K_0}{\delta}}}}{k+K_0},
    \end{align*}
    where $E_\ell$ is as defined in \eqref{def:Goodset}.

    (c) If $\max_{\ell < k}{\bE\sqbr{\Delta_\ell}} \leq \Delta_{\max}$, then
    \begin{align*}
        \bE\sqbr{\sum_{\ell=0}^{k-1}{\alpha_\ell \zeta_{\ell+1,k-1} \angl{\nabla f(x_\ell), d_\ell}}} \leq  \frac{D_1 + D_2 \Delta_{\max}}{k+K_0},
    \end{align*}
\end{lemma}
\begin{proof}
    \allowdisplaybreaks
    \textbf{Proof of (a)}. Decomposing $d_\ell$, we obtain the following:
    \begin{align*}
        &\sum_{\ell=0}^{k-1}{\alpha_\ell \zeta_{\ell+1,k-1} \angl{\nabla f(x_\ell), d_\ell}} =\sum_{\ell=0}^{k-1}{\alpha_\ell \zeta_{\ell+1,k-1} \angl{\nabla f(x_\ell), V(x_\ell,Z_{\ell+1}) - V(x_\ell,Z_\ell)} } \\
        &\leq \underbrace{-\alpha_0 \zeta_{1,k-1} \angl{\nabla f(x_0),V(x_0,Z_0)} + \alpha_{k-1} \angl{\nabla f(x_{k-1}), V(x_{k-1},Z_k)}}_{(P)} \\
        &\quad + \underbrace{\sum_{\ell=0}^{k-2}{\br{\alpha_\ell \zeta_{\ell+1,k-1} - \alpha_{\ell+1} \zeta_{\ell+2,k-1}} \angl{\nabla f(x_\ell),V(x_\ell,Z_{\ell+1})}}}_{(Q)} \\
        &\quad + \underbrace{\sum_{\ell=0}^{k-2}{\alpha_{\ell+1} \zeta_{\ell+2,k-1}\br{\angl{\nabla f(x_\ell),V(x_\ell,Z_{\ell+1})}  - \angl{\nabla f(x_{\ell+1}),V(x_{\ell+1},Z_{\ell+1})}}}}_{(R)}.
    \end{align*}
    We derive bounds on $(P)$, $(Q)$ and $(R)$ separately.
    
    \textbf{Bounding} $(P)$:
    \begin{align}
        &- \alpha_0 \zeta_{1,k-1} \angl{\nabla f(x_0),V(x_0,Z_0)} + \alpha_{k-1} \angl{\nabla f(x_{k-1}), V(x_{k-1},Z_k)}  \notag \\
        &\stackrel{(a)}{\leq} \frac{a}{K_0} \cdot \frac{K_0+1}{k+K_0} \abs{\angl{\nabla f(x_0),V(x_0,Z_0)}} + \frac{a}{k+K_0-1} \abs{\angl{\nabla f(x_{k-1}), V(x_{k-1},Z_k)}}  \notag\\
        &\stackrel{(b)}{\leq} \frac{2 a \mfr_1 \tmix L \sqrt{d}}{k+K_0} \Delta_{k-1} + \frac{2a \tmix \sqrt{d} \br{\mfr_1 L \Delta_0 + 2 \mfr_2}}{k+K_0} \notag\\
        &\stackrel{(c)}{\leq} \frac{2 a \mfr_1 \tmix L \sqrt{d}}{k+K_0} \Delta_{\max} + \frac{2 a \tmix \sqrt{d} \br{\mfr_1 L \Delta_0 + 2 \mfr_2}}{k+K_0}, \label{bdd:A}
    \end{align}
    where $(a)$ follows from Lemma~\ref{lem:zeta_bdd}, $(b)$ follows from Lemma~\ref{lem:supporting_bdd_1}, and $(c)$ follows from the assumed bound on $\Delta_{k-1}$.
    
    \textbf{Bounding} $(Q)$:
    \begin{align}
        &\sum_{\ell=0}^{k-2}{\br{\alpha_\ell \zeta_{\ell+1,k-1} - \alpha_{\ell+1} \zeta_{\ell+2,k-1}} \angl{\nabla f(x_\ell),V(x_\ell,Z_{\ell+1})}} \notag \\
        &\stackrel{(a)}{\leq} 2 \tmix \sqrt{d} \frac{\mu a - 1}{a} \sum_{\ell=0}^{k-2}{\alpha_{\ell+1}^2 \zeta_{\ell+2,k-1}  \br{\mfr_1 L \Delta_\ell + \mfr_2}} \notag \\
        &\stackrel{(b)}{\leq} \frac{2 e a \mfr_1 \tmix L \sqrt{d}}{k+K_0} \Delta_{\max} + \frac{2 e a \mfr_2 \tmix \sqrt{d}}{k+K_0}, \label{bdd:B}
    \end{align}
    where $(a)$ follows from Lemma~\ref{lem:bddalphazeta} and Lemma~\ref{lem:supporting_bdd_1}, $(b)$ follows from the assumed bound on $\Delta_\ell$ and from Lemma~\ref{lem:bddalphazeta}.
    
    \textbf{Bounding} $(R)$:
    \begin{align}
        &\sum_{\ell=0}^{k-2}{\alpha_\ell \zeta_{\ell+1,k-1}\br{\angl{\nabla f(x_{\ell+1}),V(x_{\ell+1},Z_{\ell+1})} - \angl{\nabla f(x_\ell),V(x_\ell,Z_{\ell+1})}}} \notag \\
        &\stackrel{(a)}{\leq} \tmix (L + L_g) \sqrt{d} \br{\mfr_3 L (\Delta_{\ell} + \Delta_{\ell+1}) + \mfr_4} \sum_{\ell=0}^{k-2}{\alpha_\ell^2 \zeta_{\ell+1,k-1}} \notag \\
        &\stackrel{(b)}{\leq} \frac{e a^2 \mfr_3 \tmix (L + L_g) L\sqrt{d}}{\mu a - 1} \frac{\Delta_{\max}}{k + K_0} + \frac{e a^2 \mfr_4 \tmix (L + L_g) \sqrt{d}}{\mu a - 1} \frac{1}{k + K_0}, \label{bdd:C}
    \end{align}
    where $(a)$ follows from Lemma~\ref{lem:supporting_bdd_2} and from the assumed bound on $\Delta_\ell$, $(b)$ follows from Lemma~\ref{lem:bddalphazeta}.~Combining \eqref{bdd:A}, \eqref{bdd:B}, \eqref{bdd:C}, we obtain what we want to proof.

    \textbf{Proof of (b)}. From Lemma~\ref{lem:Delta_const_bd}, we have that on the event $E_\ell$, $\Delta_\ell \leq \Delta_0 + \frac{\Gamma_1 + 2 \Gamma_2 \log{\br{\frac{2 K_0}{\delta}}}}{K_0}$. Replacing that in (a) yields what we want to show.

    \textbf{Proof of (c)}. We take expectation before replacing $\Delta_\ell$ with $\Delta_{\max}$ in \eqref{bdd:A}, \eqref{bdd:B} and \eqref{bdd:C}. Then the claim follows by replacing $\bE[\Delta_\ell]$ with $\Delta_{\max}$ and adding all three bounds.
\end{proof}

\subsubsection{Bounding the squared sum of the martingale difference terms}

\begin{lemma}\label{lem:square_sum_bdd}
    Recall the martingale difference sequence $\overline{M}_{\ell+1}$ from~\eqref{Mbar}.
    \begin{align*}
        \sum_{\ell=0}^{k-1}{\overline{M}^2_{\ell+1}} \leq 16 (ae)^2 L (\tmix^2 d + 1)\br{\frac{2AL+B}{2\mu a - 3} \frac{\Lambda(k,\delta)}{k+K_0} + \frac{C}{2\mu a - 2}} \frac{\Lambda(k,\delta)}{\br{k+K_0}^2}.
    \end{align*}
\end{lemma}
\begin{proof}
    From Lemma~\ref{lem:ub_grad}, we have $\norm{\nabla f(x_\ell)}^2 \leq 2L \Delta_\ell$.~See that,
    \begin{align*}
        \norm{M_{\ell+1}}^2 &= \norm{G_\ell - g(x_\ell, Z_\ell)}^2 \\
        &\stackrel{(a)}{\leq} 2 \br{\norm{G_\ell}^2 + \norm{g(x_\ell, Z_\ell)}^2} \\
        &\stackrel{(b)}{\leq} 4 \br{A \norm{\nabla f(x_\ell)}^2 + B \br{f(x_\ell) - f\ust} + C} \\
        &\stackrel{(c)}{\leq} 4 \br{(2AL + B) \Delta_\ell + C}.
    \end{align*}
    $(a)$ follows from the parallelogram law in a normed vector space. $(b)$ follows from Assumption~\ref{assum:martingale} and~\ref{assum:lipg}, and $(c)$ follows from Lemma~\ref{lem:ub_grad}. Also, from~\Cref{cor:mart_norm_bound}, we have $\norm{\widetilde{M}_{\ell+1}}^2 \leq 4 \tmix^2 d \br{(2 A L + B) \Delta_\ell + C}$. Hence, 
    \begin{align*}
        \norm{\nabla f(x_\ell)}^2 \norm{M_{\ell+1} + \widetilde{M}_{\ell+1}}^2 &\leq 2 \norm{\nabla f(x_\ell)}^2 \br{\norm{M_{\ell+1}}^2 + \norm{\widetilde{M}_{\ell+1}}^2} \\
        &\leq 16 L (\tmix^2 d + 1) \br{(2 A L + B) \Delta_\ell^2 + C \Delta_\ell}.
    \end{align*}
    For convenience, denote $\kappa_1 =  16 L (2 A L + B) (\tmix^2 d + 1)$ and $\kappa_2 =  16 L C (\tmix^2 d + 1)$. Recall that on the event $E_\ell$, $\Delta_\ell \leq \frac{\Lambda(\ell,\delta)}{\ell+K_0}$. Hence,
    \begin{align}\label{bdd:normGradfnormM}
        \norm{\nabla f(x_\ell)}^2 \norm{M_{\ell+1} + \widetilde{M}_{\ell+1}}^2 \ind{E_{\ell}} &\leq \kappa_1 \br{\frac{\Lambda(\ell,\delta)}{\ell+K_0}}^2 + \kappa_2 \frac{\Lambda(\ell,\delta)}{\ell+K_0}.
    \end{align}
    Using the Cauchy-Schwarz inequality and from \eqref{bdd:normGradfnormM} we have that for any $k \in \{1,2,\ldots\}$, we have
    \begin{align}
        \overline{M}_{\ell+1}^2 &\leq \alpha_\ell^2 \zeta_{\ell+1,k-1}^2 \norm{\nabla f(x_\ell)}^2 \norm{M_{\ell+1} + \widetilde{M}_{\ell+1}}^2 \ind{E_\ell} \notag\\
        &\leq \frac{(ae)^2 \kappa_1 \Lambda(\ell,\delta)^2}{(k+K_0)^{2 \mu a}} (\ell+K_0)^{2 \mu a - 4} + \frac{(ae)^2 \kappa_2 \Lambda(\ell,\delta)}{(k+K_0)^{2 \mu a}} (\ell+K_0)^{2 \mu a -3}. \label{ub:sqr}
    \end{align}
    Summing \eqref{ub:sqr} from $\ell = 0$ to $k - 1$, we get the following bound:
    \begin{align*}
        &\sum_{\ell=0}^{k-1}{\overline{M}_{\ell+1}^2} \leq \frac{(ae)^2 \kappa_1 \Lambda(k,\delta)^2}{(k+K_0)^{2 \mu a}} \sum_{\ell=0}^{k-1}{\br{\ell+K_0}^{2\mu a - 4}} + \frac{(ae)^2 \kappa_2 \Lambda(k,\delta)}{(k+K_0)^{2 \mu a}} \sum_{\ell=0}^{k-1}{\br{\ell+K_0}^{2\mu a - 3}} \\
        &\leq \frac{(ae)^2 \kappa_1 \Lambda(k,\delta)^2}{(k+K_0)^{2 \mu a}} \int_{0}^{k}{\br{y+K_0}^{2\mu a - 4} dy} + \frac{(ae)^2 \kappa_2 \Lambda(k,\delta)}{(k+K_0)^{2 \mu a}} \int_{0}^{k}{\br{y+K_0}^{2\mu a - 3} dy} \\
        &\leq \frac{(ae)^2 \kappa_1}{2\mu a - 3} \frac{\Lambda(k,\delta)^2}{\br{k+K_0}^3} + \frac{(ae)^2 \kappa_2}{2\mu a - 2} \frac{\Lambda(k,\delta)}{\br{k+K_0}^2} \\
        &= (ae)^2 \br{\frac{\kappa_1}{2\mu a - 3} \frac{\Lambda(k,\delta)}{k+K_0} + \frac{\kappa_2}{2\mu a - 2}} \frac{\Lambda(k,\delta)}{\br{k+K_0}^2}.
    \end{align*}
    This concludes the proof.
\end{proof}

\subsubsection{Choosing \texorpdfstring{$\Gamma_2$}{Gamma}}

\begin{lemma}\label{lemma:Gamma2}
    Let the lower bound on $K_0$ stated in \eqref{lb:K_0} hold. Then, for every $k$,
    \begin{align*}
        \Gamma_2 \log\left(\frac{2k}{\delta}\right) \geq \sqrt{2 \br{\nu_1 + \frac{\nu_2 \Gamma_2}{K_0} \log\br{\frac{2K_0}{\delta}}} \Lambda(k,\delta) \log{\br{\frac{2k}{\delta}}}}.
    \end{align*}
    where $\Gamma_2$ as defined in \eqref{def:Gm2}.
\end{lemma}
\begin{proof}
    Consider the inequality
    \begin{align}\label{ineq:dummygamma2}
        x \log\left(\frac{2k}{\delta}\right) &\geq \sqrt{2 \br{\nu_1 + \frac{\nu_2 x}{K_0} \log\br{\frac{2K_0}{\delta}}}} \\
        &\quad \times \sqrt{\br{K_0 \Delta_0 + \Gamma_1 + x \br{\log{\br{\frac{2K_0}{\delta}}} + \log{\br{\frac{2k}{\delta}}}}} \log{\br{\frac{2k}{\delta}}}}. \notag
    \end{align}
    Squaring both sides and simplifying, we obtain the quadratic inequality for $x$, $\mathfrak{p} x^2 - \mathfrak{q} x - \mathfrak{r} \geq 0$, where
    \begin{align*}
        \mathfrak{p} &=  \log\br{\frac{2k}{\delta}} \br{1 - \frac{2 \nu_2}{K_0} \log\br{\frac{2K_0}{\delta}} \br{1 + \frac{\log\br{\frac{2K_0}{\delta}}}{\log\br{\frac{2k}{\delta}}}}}, \\
        \mathfrak{q} &= 2 \br{\nu_1 \log\br{\frac{2k}{\delta}} + \br{\nu_1 + \nu_2 \br{\Delta_0 + \frac{\Gamma_1}{K_0}}} \log\br{\frac{2K_0}{\delta}}}, \mbox{ and} \\
        \mathfrak{r} &= 2 \nu_1 \br{K_0 \Delta_0 + \Gamma_1}.
    \end{align*}
    Note that any $x$ satisfying $x \geq \frac{\mathfrak{q}}{\mathfrak{p}} + \sqrt{\frac{\mathfrak{r}}{\mathfrak{p}}}$ satisfies \eqref{ineq:dummygamma2}, we will show that $\Gamma_2 \geq  \frac{\mathfrak{q}}{\mathfrak{p}} + \sqrt{\frac{\mathfrak{r}}{\mathfrak{p}}}$.

    It follows from the lower bound on $K_0$~\eqref{lb:K_0} and from Lemma~\ref{lem:lb_K_0} that
    \begin{align*}
         K_0 \geq 4 \nu_2 \log\br{\frac{2 K_0}{\delta}} \br{1+ \frac{\log\br{\frac{2 K_0}{\delta}}}{\log\br{\frac{2}{\delta}}}},
    \end{align*}
    which results into the following lower bound for $\mathfrak{p}$: $\mathfrak{p} \geq \frac{1}{2}\log\br{\frac{2k}{\delta}}$. Also, from the definition of $\nu_1$ and $\nu_2$ and using $k\geq 1$, we can write that
    \begin{align*}
        \mathfrak{q} &\leq 2 \nu_1 \log\br{\frac{2k}{\delta}} \br{1 + 3 \frac{\log\br{\frac{2K_0}{\delta}}}{\log\br{\frac{2}{\delta}}}}.
    \end{align*}
    
    Thus, 
    \begin{align*}
        x \geq 4 \nu_1 \log\br{\frac{2k}{\delta}} \br{1 + 3 \frac{\log\br{\frac{2K_0}{\delta}}}{\log\br{\frac{2}{\delta}}}} + 2\sqrt{\nu_1 \br{\Delta_0 \frac{K_0}{\log\br{\frac{2}{\delta}}} + \Gamma_1}}
    \end{align*}
    satisfies \eqref{ineq:dummygamma2}, which proves the claim.
\end{proof}

\section{Proof of Theorem~\ref{thm:expectation_markov}}\label{sec:expt_proof}
\begin{theorem}[Detailed version of Theorem~\ref{thm:expectation_markov}]\label{thm:expectation_markov_full}
    Let the objective $f:\bR^d \to \bR$ be $\mu$-P\L~\eqref{eq:pl_cond}. Under Assumptions~\ref{assum:smoothness}, \ref{assum:martingale}, \ref{assum:Markov}, \ref{assum:lipg}, and \ref{assum:expt_noisy_grad_norm}, and for $K_0 \geq \max\flbr{a L \br{2 A + \frac{B}{\mu}}, \mu a, 2D_2}$, the expected suboptimality of SGD iterates satisfies
    \begin{align*}
        \bE\sqbr{\Delta_k} &\leq \frac{\Delta_0 (K_0 + 2 D_2) + ea^2 CL + 2 D_1}{k + K_0}, \mbox{ for every } k \in \{1,2,\ldots\},
    \end{align*}
    where $D_1$ and $D_2$ are as defined in \eqref{def:D_1} and \eqref{def:D_2}, respectively.
\end{theorem}

\begin{proof}
    We prove the claim by mathematical induction. Note that the base case (for $k=0$) is trivially true. Let us assume that for $\ell = 0$ to $k-1$,
    \begin{align*}
        \bE\sqbr{\Delta_\ell} &\leq \frac{\Delta_0 (K_0 + 2 D_2) + ea^2 CL + 2 D_1}{\ell + K_0}.
    \end{align*}
    Hence, we can also write that for every $\ell \in \{0, 1, \ldots, k-1\}$,
    \begin{align}
        \bE\sqbr{\Delta_\ell} &\leq \frac{\Delta_0 (K_0 + 2 D_2) + ea^2 CL + 2 D_1}{K_0}. \label{crude_bdd:expt_Delta}
    \end{align}
    Taking an expectation on both sides of \Cref{eq:iter_Delta}, we obtain that
    \begin{align}
        \bE\sqbr{\Delta_k} &\leq \Delta_0 \zeta_{0,k-1} + \frac{C L}{2} \sum_{\ell=0}^{k-1}{\alpha_\ell^2 \zeta_{\ell+1,k-1}} + \bE\sqbr{\sum_{\ell=0}^{k-1}{{\alpha_\ell \zeta_{\ell+1,k-1} \langle \nabla f(x_\ell), d_\ell \rangle}}} \notag\\
        &\quad - \bE\sqbr{\sum_{\ell=0}^{k-1}\alpha_\ell \zeta_{\ell+1,k-1}
        \left\langle \nabla f(x_\ell), M_{\ell+1} + \widetilde{M}_{\ell+1}\right\rangle} \notag\\
        &= \Delta_0 \zeta_{0,k-1} + \frac{C L}{2} \sum_{\ell=0}^{k-1}{\alpha_\ell^2 \zeta_{\ell+1,k-1}} + \bE\sqbr{\sum_{\ell=0}^{k-1}{{\alpha_\ell \zeta_{\ell+1,k-1} \langle \nabla f(x_\ell), d_\ell \rangle}}}.\label{eq:iter_Delta_expt}
    \end{align}
    Using Lemma~\ref{lem:bdd_determ_terms}, Lemma~\ref{lem:bdd_markov_residue} and \Cref{crude_bdd:expt_Delta}, we obtain
    \begin{align*}
        \bE\sqbr{\Delta_k} &\leq \frac{K_0 \Delta_0 + e a^2 C L/ 2 + D_1 + D_2 \frac{\Delta_0 (K_0 + 2 D_2) + ea^2 CL + 2 D_1}{K_0}}{k+K_0} \\
        &\stackrel{(a)}{\leq} \frac{\Delta_0 (K_0 + 2 D_2) + ea^2 CL + 2 D_1}{k+K_0},
    \end{align*}
    where $(a)$ follows from the condition that $K_0 \geq 2 D_2$. Thus, the proof is complete. 
\end{proof}

\section{Properties of the solution of the Poisson equation} \label{sec:pois_sol}

In this section, we derive properties of the solution of the Poisson equation~\eqref{eq:poisson}, i.e.,
\begin{align*}
    V(x,z) \coloneqq \bE\sqbr{\sum_{\ell=0}^{\infty}{g(x,Z_\ell) - \nabla f(x)} \mid Z_0 = z}.
\end{align*}

\subsection{Bound on the norm of \texorpdfstring{$V(x,\cdot)$}{V(x,.)} and the derived martingale difference term}
\begin{lemma}\label{lem:pois_sol_norm_bound}
    Under Assumption~\ref{assum:Markov} and \ref{assum:lipg}, for every $x \in \bR^d$, $$\norm{V(x,z)} \leq 2 \tmix \sqrt{d} \sup_{z\up}{\norm{g(x,z\up)}},$$ and $$\norm{V(x,z)}^2 \leq 4 \tmix^2 d \br{(2 A L + B) \br{f(x) - f\ust} + C},~\forall z \in \cZ.$$
\end{lemma}

\begin{proof}
    We start from the definition of $V(x,z)$~\eqref{def:poisson_sol} as follows.
    \begin{align}
        \norm{V(x,z)} &= \norm{\sum_{k=0}^{\infty}{\int{g(x,z\up) \br{p\uc{k}(z\up \mid z) - \pi(z\up)} dz\up}}} \notag\\
        &\stackrel{(a)}{\leq} \sqrt{d} \sup_{z\up}{\norm{g(x,z\up)}} \sum_{k=0}^{\infty}{\norm{p\uc{k}(\cdot \mid z) - \pi}_{TV}} \notag\\
        &\stackrel{(b)}{\leq} \sqrt{d} \sup_{z\up}{\norm{g(x,z\up)}} \tmix \sum_{k=0}^{\infty}{\tfrac{1}{2^k}} \notag\\
        &= 2 \tmix \sqrt{d} \sup_{z\up}{\norm{g(x,z\up)}}, \label{bdd:Vnorm}
    \end{align}
    where $(a)$ follows from Lemma~\ref{lem:bdd_dotdifLv} and $(b)$ follows from the definition of $\tmix$~\eqref{def:tmix}. Squaring both sides of \eqref{bdd:Vnorm}, we obtain
    \begin{align}
        \norm{V(x,z)}^2 &\leq 4 \tmix^2 d \br{\sup_{z\up}{\norm{g(x,z\up)}}}^2 \notag\\
        &\stackrel{(a)}{\leq} 4 \tmix^2 d \br{A \norm{\nabla f(x)}^2 + B \br{f(x) - f\ust} + C} \notag\\
        &\stackrel{(b)}{\leq} 4 \tmix^2 d \br{(2 A L + B) \br{f(x) - f\ust} + C}, \label{bdd:Vnormsqr}
    \end{align}
    where $(a)$ follows from Assumption~\ref{assum:lipg}, and $(b)$ follows from Lemma~\ref{lem:ub_grad}.
\end{proof}

\begin{corollary}\label{cor:mart_norm_bound}
    Recall the martingale difference term $\widetilde{M}_{k+1}$ from Section~\ref{sec:proof_outline},
    \begin{align*}
        \widetilde{M}_{k+1} &\coloneqq V(x_k,Z_{k+1}) - \int{V(x_k,z) p(\rmd z \mid Z_k)}.
    \end{align*}
    We have that, for every $k \geq 0$ on every sample path, 
    \[
    \norm{\widetilde{M}_{k+1}}^2 \leq 16 \tmix^2 d \br{(2 A L + B)\br{f(x_k) - f\ust} + C}
    \] 
\end{corollary}
\begin{proof}
    Note that $\norm{\widetilde{M}_{k+1}}^2 \leq 4 \max_{z \in \cZ}{\norm{V(x_k,z)}^2}$. Then, the desired result follows from Lemma~\ref{lem:pois_sol_norm_bound}.
\end{proof}

\subsection{Lipschitz continuity of \texorpdfstring{$V(x,\cdot)$}{V(x,.)}}
\begin{lemma}\label{lem:pois_sol_lipschitz}
    If Assumption~\ref{assum:Markov} and \ref{assum:lipg} hold true, $V(x,\cdot)$ is $2 \tmix L_g \sqrt{d}$-Lipschitz, i.e., $\norm{V(x,z) - V(x\up,z)} \leq 2 \tmix L_g \sqrt{d} \norm{x - x\up}$ for every $x,x\up \in \bR^d, z \in \cZ$.
\end{lemma}
\begin{proof}
    We start from the definition of $V(x,z)$~\eqref{def:poisson_sol} as follows.
        \begin{align}
            &\norm{V(x,z) - V(x\up,z)} \notag\\
            &= \norm{\bE\sqbr{\sum_{\ell=0}^{\infty}{\br{g(x,Z_\ell)-\nabla f(x)} - {g(x\up,Z_\ell) - \nabla f(x\up)} \mid Z_0 = z}}} \notag\\
            &= \norm{\sum_{\ell=0}^{\infty}{\int{\br{p\uc{\ell}(\rmd z\up \mid z) - \pi(\rmd z\up)} \br{g(x,z\up) - g(x\up,z\up)}}}} \notag\\
            &\stackrel{(a)}{\leq} \sqrt{d} \sup_{z\up}{\norm{g(x,z\up) - g(x\up,z\up)}} \br{\sum_{\ell=0}^{\infty}{\norm{p\uc{\ell}(\cdot \mid z) - \pi}_{TV}}} \notag\\
            &\stackrel{(b)}{\leq} 2 \tmix L_g \sqrt{d} \norm{x-x\up}, \label{lipV}
        \end{align}
        where $(a)$ follows from \Cref{lem:bdd_dotdifLv}, $(b)$ follows from Assumption~\ref{assum:lipg}~\eqref{eq:lipg} and from the definition of $\tmix$~\eqref{def:tmix}.
\end{proof}

\subsection{Analysis of \texorpdfstring{$\angl{\nabla f(x), V(x,\cdot)}$}{V(x,.)}}
\begin{lemma}\label{lem:supporting_bdd_1}
    Let Assumption~\ref{assum:Markov} and \ref{assum:lipg} hold true. Then, for every $x \in \bR^d, z \in \cZ$, 
    $$\angl{\nabla f(x), V(x,z)} \leq \mfr_1 \tmix L \sqrt{d} (f(x) - f\ust) + \mfr_2 \tmix \sqrt{d},$$
    where $\mfr_1, \mfr_2$ are constants that depends on $A, B$ and $C$.
\end{lemma}
\begin{proof}
    We use the Cauchy-Schwarz inequality in order to obtain the desired inequality.
    \begin{align*}
        \angl{\nabla f(x), V(x,z)} &\leq \norm{\nabla f(x)} \norm{V(x,z)} \\
        &\stackrel{(a)}{\leq} 2 \tmix \sqrt{d} \norm{\nabla f(x)} \sup_{z\up}{\norm{g(x,z\up)}} \\
        &\stackrel{(b)}{\leq} 2 \tmix \sqrt{d} \sqrt{2L (f(x) - f\ust) \br{(2A L + B) (f(x) - f\ust) + C}} \\
        &\leq 2 \tmix \sqrt{d} \br{\sqrt{2L(2A L + B)} (f(x) - f\ust) + \sqrt{\frac{L {C}^2}{2(2A L + B)}}},
    \end{align*}
    where $(a)$ follows from \Cref{bdd:Vnorm} and $(b)$ follows from \Cref{lem:ub_grad} and from Assumption~\ref{assum:lipg}.
\end{proof}

\begin{lemma}\label{lem:supporting_bdd_2}
    Consider the iterates $\{x_k\}$ that are generated by the SGD iteration~\eqref{iter:sgd}, and $\{Z_k\}$, the Markov chain associated with the Markov noise. Under Assumption~\ref{assum:Markov} and \ref{assum:lipg}, there exists constants $\mfr_3$ and $\mfr_4$ that depends on $A, B$ and $C$, such that for every $k$,
    \begin{multline*}
        \abs{\angl{\nabla f(x_{k+1}),V(x_{k+1},Z_{k+1})} - \angl{\nabla f(x_k),V(x_k,Z_{k+1})}} \\ \leq \alpha_k \tmix (L + L_g) \sqrt{d} \br{\mfr_3 L \br{\Delta_k + \Delta_{k+1}} + \mfr_4}.
    \end{multline*}
\end{lemma}
\begin{proof}
    First, let us derive a point-wise bound.
    \begin{align*}
        &\abs{\angl{\nabla f(x),V(x,z)} - \angl{\nabla f(x\up),V(x\up,z)}} \\
        &\stackrel{(a)}{\leq} \abs{\angl{\nabla f(x),V(x,z)-V(x\up,z)}} + \abs{\angl{\nabla f(x)-\nabla f(x\up),V(x\up,z)}} \\
        &\stackrel{(b)}{\leq} \norm{\nabla f(x)} \norm{V(x,z)-V(x\up,z)} + \norm{\nabla f(x)-\nabla f(x\up)} \norm{V(x\up,z)} \\
        &\stackrel{(c)}{\leq} \br{2 \tmix L_g \sqrt{d} \norm{\nabla f(x)} + L \norm{V(x\up,z)}} \norm{x-x\up} \\
        &\stackrel{(d)}{\leq} 2 \tmix \sqrt{d} \br{L_g \norm{\nabla f(x)} + L \sup_{z\up}\norm{g(x\up,z)}} \norm{x-x\up}.
    \end{align*}
    Inquality $(a)$ and $(b)$ follow from the triangle inequality and the Cauchy-Schwarz inequality, respectively; $(c)$ follows from $L$-smoothness of $f$ (Assumption~\ref{assum:smoothness}) and from Lemma~\ref{lem:pois_sol_lipschitz}. Now, replacing $x,x\up$ and $z$ by $x_{k+1}, x_k$ and $Z_{k+1}$, respectively, in the above, we obtain
    \begin{align*}
        &\abs{\angl{\nabla f(x_{k+1}),V(x_{k+1},Z_{k+1})} - \angl{\nabla f(x_k),V(x_k,Z_{k+1})}} \\
        &\leq 2 \tmix \sqrt{d} \br{L_g \norm{\nabla f(x_{k+1})} + L \sup_{z\up}{\norm{g(x_k,z\up)}}} \norm{\alpha_k G_k} \\
        &\stackrel{(a)}{\leq} 2 \alpha_k \tmix \sqrt{d} \br{L_g \sqrt{2L \Delta_{k+1}} + L \sqrt{(2A L + B) \Delta_k + C}} \sqrt{(2AL +B) \Delta_k + C} \\
        &= \alpha_k \tmix (L + L_g) \sqrt{d} \br{\mfr_3 L \max\{\Delta_k,\Delta_{k+1}\} + \mfr_4},\\
        &\stackrel{(b)}{\leq} \alpha_k \tmix (L + L_g) \sqrt{d} \br{\mfr_3 L \br{\Delta_k + \Delta_{k+1}} + \mfr_4}
    \end{align*}
    where $(a)$ follows from Assumption~\ref{bdd:noisy_grad_norm} and \ref{assum:lipg} and from Lemma~\ref{lem:ub_grad}. $(b)$ follows from the fact that $\Delta_\ell > 0$ for every $\ell$.
\end{proof}
\section{Proof of Theorem~\ref{thm:concentration_mart}}\label{sec:mart_noise_proof}

\begin{theorem}[Detailed version of Theorem~\ref{thm:concentration_mart}.] \label{thm:concentration_mart_full}
    Denote
    \begin{align*}
        \tnu_1 &\coloneqq 8L (ae)^2 \br{\frac{2L(A+1) + B}{2\mu a - 3} \Delta_0 + \frac{C}{\mu a - 1}} + \frac{e a^2 C L}{16}\\
        \tnu_2 &\coloneqq 8L (ae)^2 \frac{2L(A+1) + B}{2\mu a - 3}.
    \end{align*}
    Let the objective $f:\bR^d \to \bR$ be $\mu$-P\L~\eqref{eq:pl_cond}, and let 
    \begin{align*}
        K_0 \geq \max\flbr{\frac{a L}{2} \br{2 A + \frac{B}{\mu}}, \mu a, 8 \tnu_2 \log\br{\frac{16 \tnu_2}{\delta}}}.
    \end{align*}
    Suppose the noisy gradient samples are free of Markov noise, i.e., $G_k = \nabla f(x_k) + M_{k+1}$ for all $k=0,1,\ldots$, and Assumption~\ref{assum:smoothness}, \ref{assum:martingale} and \ref{assum:noisy_grad_norm} hold true. Then sub-optimalities of SGD iterates are bounded on a set of probability at least $1 - \delta$, where $\delta \in (0,1)$, as follows: for every $k \in \{1,2,\ldots\}$,
    \begin{align*}
        \Delta_k &\leq \frac{\Delta_0 K_0 + \tGm_1 + \tGm_2 \log{\br{\frac{2k}{\delta}}}}{k + K_0},
    \end{align*}
    where 
    \begin{align}
        \tGm_1 &\coloneqq \frac{ea^2 C L}{2}, \label{def:tGm1} \\
        \tGm_2 &\coloneqq 12 \tnu_1 \br{1 + \log\br{8 \tnu_1}} + 2 \sqrt{\tnu_1 \br{\overline{K}_0 \Delta_0 + e a^2 CL}}. \label{def:tGm2}
    \end{align}
    $\overline{K}_0$ be a constant such that $\overline{K}_0 \geq K_0 / log(2/\delta)$.
\end{theorem}
\begin{proof}
    Following the same steps as in the proof of Lemma~\ref{lem:iter_Delta} and from Lemma~\ref{lem:bdd_determ_terms} we obtain the following inequality:
    \begin{align} \label{eq:iter_Delta_mart}
        \Delta_k &\leq \frac{K_0 \Delta_0 + \tGm_1}{k + K_0} - \sum_{\ell=0}^{k-1}{{\alpha_\ell \zeta_{\ell+1,k-1} \langle \nabla f(x_\ell), M_{\ell+1}\rangle}}.
    \end{align}
    Define the event,
    \begin{align}\label{def:Goodset_mart}
        F_k \coloneqq \flbr{\Delta_k \leq \frac{K_0 \Delta_0 + \ind{k\neq 0} \br{\tGm_1 + \tGm_2 \log{\br{\frac{2k}{\delta}}}}}{k+K_0}},~k=0,1,\ldots,
    \end{align}
    where $\tGm_1$ and $\tGm_2$ are defined in \eqref{def:tGm1} and \eqref{def:tGm2}, respectively. For convenience, denote $\widehat{M}_{\ell+1} := - \alpha_\ell \zeta_{\ell+1,k-1} \langle \nabla f(x_\ell), M_{\ell+1} \rangle \ind{F_\ell}$. Following the same steps as in Lemma~\ref{lem:keylemma}, we have that
    \begin{align}
        \bP\br{\cup_{k=0}^{\infty}{\bar{F}_k}} &\leq \sum_{k=1}^{\infty}{\bP\br{\flbr{\sum_{\ell=0}^{k-1}{\widehat{M}_{\ell+1}} > \frac{\tGm_2 \log{\br{\frac{2k}{\delta}}}}{k+K_0}}}} \label{ineq:keybound_mart}
    \end{align}
    In order to use Azuma-Hoeffding inequality~(Lemma~\ref{lem:ah_ineq}), we need to bound
    \begin{align*}
        \sum_{\ell=0}^{k-1} \left(\alpha_\ell \zeta_{\ell+1,k-1} \langle \nabla f(x_\ell), M_{\ell+1} \rangle \ind{F_\ell}\right)^2.
    \end{align*}
    Using the Cauchy-Schwarz inequality, we have
    \begin{align*}
        \sum_{\ell=0}^{k-1} \widehat{M}_{\ell+1}^2 \leq \sum_{\ell=0}^{k-1} \alpha_\ell^2 \zeta_{\ell+1,k-1}^2 \norm{\nabla f(x_\ell)}^2 \norm{M_{\ell+1}}^2 \ind{F_\ell}.
    \end{align*}
    From Lemma~\ref{lem:ub_grad}, we have $\norm{\nabla f(x_\ell)}^2 \leq 2L \Delta_\ell$.~See that,
    \begin{align*}
        \norm{M_{\ell+1}}^2 &= \norm{G_\ell - \nabla f(x_\ell)}^2 \\
        &\stackrel{(a)}{\leq} 2 \br{\norm{G_\ell}^2 + \norm{\nabla f(x_\ell)}^2} \\
        &\stackrel{(b)}{\leq} 2\br{(A+1) \norm{\nabla f(x_\ell)}^2 + B \br{f(x_\ell) - f\ust} + C} \\
        &\stackrel{(c)}{\leq} 2 \br{2L(A+1) + B} \Delta_\ell + 4C.
    \end{align*}
    $(a)$ follows from the parallelogram law in a normed vector space. $(b)$ follows from Assumption~\ref{assum:noisy_grad_norm}, and $(c)$ follows from Lemma~\ref{lem:ub_grad}. For ease of notation, denote $\tkp_1 := 4L \br{2L(A+1) + B}$ and $\tkp_2 := 8 L C$. Multiplying the upper bounds on $\norm{\nabla f(x_\ell)}^2$ and $\norm{M_{\ell+1}}^2$, we have
    \begin{align*}
        \norm{\nabla f(x_\ell)}^2 \norm{M_{\ell+1}}^2 \leq \tkp_1 \Delta_\ell^2 + \tkp_2 \Delta_\ell.
    \end{align*}
    Following the same steps as the proof of Lemma~\ref{lem:square_sum_bdd}, we obtain that,
    \begin{align}\label{bdd:mart_sum_square}
        \sum_{\ell=0}^{k-1} \widehat{M}_{\ell+1}^2 \leq (ae)^2 \br{\frac{\tkp_1}{2\mu a - 3} \frac{\tLm(k,\delta)}{k+K_0} + \frac{\tkp_2}{2\mu a - 2}} \frac{\tLm(k,\delta)}{\br{k+K_0}^2}.
    \end{align}
    From Lemma~\ref{lem:Delta_const_bd}, we have that
    \begin{align*}
        \frac{\tLm(k,\delta)}{k+K_0} \leq \br{\Delta_0 + \frac{\Gamma_1}{K_0}} + \frac{\Gamma_2}{K_0} \log{\br{\frac{2K_0}{\delta}}}.
    \end{align*}
    Invoking this in \Cref{bdd:mart_sum_square}, we have
    \begin{align*}
        2\sum_{\ell=0}^{k-1}{\widehat{M}_{\ell+1}^2} \leq \frac{\br{\tnu_1 + \frac{\tnu_2 \tGm_2}{K_0} \log\br{\frac{2K_0}{\delta}}} \tLm(k,\delta)}{\br{k+K_0}^2}.
    \end{align*}
    Now, using the Azuma-Hoeffding inequality (Lemma~\ref{lem:ah_ineq}), we can write that,
    \begin{align*}
        \bP\br{\flbr{\sum_{\ell=0}^{k-1}{\widehat{M}_{\ell+1}} > \eps}} \leq \exp\br{- \frac{\eps^2 \br{k+K_0}^2}{\br{\tnu_1 + \frac{\tnu_2 \tGm_2}{K_0} \log\br{\frac{2K_0}{\delta}}} \tLm(k,\delta)}}. 
    \end{align*}
    or
    \begin{align*}
        \bP\left(\left\{\sum_{\ell=0}^{k-1}{\widehat{M}_{\ell+1}} \!>\! \frac{1}{k + K_0} \sqrt{2\br{\tnu_1 + \frac{\tnu_2 \tGm_2}{K_0} \log\br{\frac{2K_0}{\delta}}} \tLm(k,\delta) \log{\br{\frac{2k}{\delta}}}} \right\}\right) \!\leq\! \frac{6}{\pi^2} \frac{\delta}{k^2}.
    \end{align*}
    Using the lower bound on $K_0$, we can show that 
    \begin{align*}
        \bP\br{\flbr{\sum_{\ell=0}^{k-1}{\widehat{M}_{\ell+1}} > \frac{\tGm_2 \log\br{\frac{2k}{\delta}}}{k + K_0}}} &\leq \frac{6}{\pi^2} \frac{\delta}{k^2}.
    \end{align*}
    Using a union bound for $k = 1, 2, \ldots$, we have,
    \begin{align*}
        \bP\br{\flbr{\sum_{\ell=0}^{k-1}{{\alpha_\ell \zeta_{\ell+1,k-1} \langle \nabla f(x_\ell), M_{\ell+1}\rangle \ind{F_\ell}}} \geq \frac{\tGm_2 \log{\br{\frac{2k}{\delta}}}}{k + K_0} \mbox{ for every } k \in \bN}} &\leq \delta.
    \end{align*}
    The proof is then complete by invoking \eqref{ineq:keybound_mart}.
\end{proof}
\section{Proofs from Section~\ref{sec:applications} (Applications)}\label{sec:appli_proofs}

\begin{proof}[\textbf{Proof of Proposition~\ref{prop:token_lr}}]
    (a) Considering the scaling factor of $1/N$, the claim follows from \cite[p.14]{karimi2016linear}.

    (b) Noting that $\nabla \cL(\theta) = \bA^\top (\bA \theta - \bb)$, we obtain
    \begin{align*}
        \norm{\nabla \cL(\theta) - \nabla \cL(\theta\up)} &= \frac{1}{N} \norm{\bA^\top (\bA \theta - \bb) - \bA^\top (\bA \theta\up - \bb)}\\
        &\leq \frac{\sigma_{\max}^2(\bA)}{N} \norm{\theta - \theta\up}.
    \end{align*}

    (c) The proof is similar to the proof of (ii).

    (d) Using the smoothness property~(Lemma~\ref{lem:ub_grad}) of $\tilde{\cL}$ and the P\L ~inequality for $\cL$, we obtain
    \begin{align*}
        \norm{\nabla \tilde{\cL}(\theta;i)}^2 &\leq 2 \frac{\sigma_{\max}^2(\bA\uc{i})}{N_i} \tilde{\cL}(\theta;i) \\
        &= \frac{\sigma_{\max}^2(\bA\uc{i})}{N_i^2} \norm{\bA\uc{i} \theta - \bb\uc{i}}^2 \\
        &\leq 2 N \frac{\sigma_{\max}^2(\bA\uc{i})}{N_i^2} \cL(\theta).
    \end{align*}
\end{proof}

\begin{proof}[\textbf{Proof of Lemma~\ref{lem:markov_dpsgd}}]
    Note that $\{\zeta_k(i)\}$ are independent and identical Markov chains for each $i \in \{1,2,\ldots, N\}$. Let us denote the corresponding common stationary distribution by $\tilde{\pi}$. Also, denote $S = \sum_{j=1}^{N}{\ind{\zeta(j) = b-1} (\phi(w;\bA(j)) - \bb(j))^2}$, and $T = \sum_{j=1}^{N}{\ind{\zeta(j) = b-1}}$. Then,
    \begin{align*}
        \bE_\pi\sqbr{\tilde{\cL}(w;\zeta)} &= \bE_\pi\sqbr{S/T} = \sum_{\ell=0}^{N}{\bP(T = \ell) ~\bE_\pi\sqbr{(S/T) \mid T = \ell}}.
    \end{align*}
    Conditioned on $T=\ell$, exactly $\ell$ indices are selected uniformly at random. Hence,
    \begin{align*}
        \bE_\pi\sqbr{\tilde{\cL}(w;\zeta)} &= \sum_{\ell=0}^{N}{\bP(T = \ell) \frac{1}{N}\sum_{j=1}^{N}{(\phi(w;\bA(j)) - \bb(j))^2}} \\
        &= \frac{1}{N} \sum_{j=1}^{N}{(\phi(w;\bA(j)) - \bb(j))^2} \sum_{\ell=0}^{N}{\binom{N}{\ell} \tilde{\pi}(b-1)^\ell (1 - \tilde{\pi}(b-1))^{N-\ell}} \\
        &= \frac{1}{N}\sum_{j=1}^{N}{(\phi(w;\bA(j)) - \bb(j))^2} \\
        &= \cL(w).
    \end{align*}
\end{proof}

\begin{proof}[\textbf{Proof of Proposition~\ref{prop:dp_sgd}}]
    \begin{enumerate}[(a)]
        \item Follows from Theorem~$1$ of \cite{liu2022loss}.
        \item Denote $N(\zeta) \coloneqq \sum_{j=1}^{N}{\ind{\zeta(j) = b-1}}$ and recall $\tilde{\ell}_i(w) = \frac{1}{2} (\phi(w;\bA(i)) - \bb(i))^2$. Therefore,
        \begin{align*}
            \norm{\nabla \tilde{\cL}(w;\zeta) - \nabla \tilde{\cL}(w\up;\zeta)} &= \norm{\frac{1}{N(\zeta)} \sum_{i=1}^{N}{\ind{\zeta(i) = b-1} \br{\nabla \tilde{\ell}_i(w) - \nabla \tilde{\ell}_i(w\up)}}} \\
            &\leq \frac{L}{N(\zeta)} \sum_{i=1}^{N}{\ind{\zeta(i) = b-1} \norm{w - w\up}} \\
            &= L \norm{w - w\up}.
        \end{align*}
        Taking $\zeta = (1, 1, \ldots , 1)^\top \in \bR^N$, we have that $\cL(\cdot)$ is $L$-Lipschitz.
        \item Noting that $\cL(w;\zeta) = \frac{1}{N(\zeta)}\sum_{i=1}^{N}{\ind{\zeta(i) = b-1} \tilde{\ell}_i(w)}$, from Lemma~\ref{lem:ub_grad}, we can write
        \begin{align*}
            \norm{\nabla \tilde{\cL}(w;\zeta)}^2 &\leq \frac{L}{N(\zeta)} \sum_{i=1}^{N}{\ind{\zeta(i) = b-1} \tilde{\ell}_i(w)} \\
            &\leq \frac{2L N}{N(\zeta)} \frac{1}{2N} \sum_{i=1}^{N}{\tilde{\ell}_i(w)} \\
            &= \frac{2 L N}{N(\zeta)} \cL(w).
        \end{align*}
        This concludes the proof.
    \end{enumerate}
\end{proof}

\begin{proof}[\textbf{Proof of Proposition~\ref{prop:stationary_dist}}]
    We denote the 2-norm of matrix $A$ by $\|A\|$. It follows from Assumption~\ref{assum:stable} that there exist constants $C>0$ and $r\in(0,1)$ such that
    \begin{align*}
        \|A\lst^k\| \le C r^k \quad \text{for all } k\ge 0.
    \end{align*}
    Unrolling the system equation from time $k$ to $0$ gives,
    \begin{align*}
        Z_k = A\lst^k Z_0 + \sum_{j=0}^{k-1} A\lst^j w_{k-1-j}.
    \end{align*}
    Consider the infinite series
    \begin{align*}
        Z_\infty := \sum_{j=0}^{\infty} A\lst^j w_{-1-j},
    \end{align*}
    where we use negative indexing of the noise sequence to accommodate infinitely many terms. We show this converges in $L^2$. Using independence and zero mean,
    \begin{align*}
        \mathbb{E}\|A\lst^j w_0\|^2 \le \|A\lst^j\|^2 \mathbb{E}\|w_0\|^2 \le C^2 r^{2j} \mathbb{E}\|w_0\|^2.
    \end{align*}
    Hence
    \begin{align*}
        \sum_{j=0}^{\infty} \mathbb{E}\|A\lst^j w_0\|^2c\le C^2 \mathbb{E}\|w_0\|^2 \sum_{j=0}^{\infty} r^{2j} < \infty.
    \end{align*}
    Thus, the series defining $Z_\infty$ converges in $L^2$. Observe that
    \begin{align*}
        A\lst Z_\infty + w_0 = \sum_{j=1}^{\infty} A\lst^j w_{-j} + w_0 = \sum_{j=0}^{\infty} A\lst^j w_{-j}.
    \end{align*}
    The right-hand side has the same distribution as $Z_\infty$.
    Hence, the law of $Z_\infty$ is stationary. Let $\tilde Z_k$ denote the stationary chain driven by the same noise sequence. Then,
    \begin{align*}
        Z_k - \tilde Z_k = A\lst^k (Z_0 - \tilde Z_0).
    \end{align*}
    Therefore
    \begin{align*}
        \|Z_k - \tilde Z_k\| \le \|A\lst^k\| \, \|Z_0 - \tilde Z_0\| \le C r^k \|Z_0 - \tilde Z_0\|.
    \end{align*}
    Taking expectations,
    \begin{align*}
        \mathbb{E}\|Z_k - \tilde Z_k\| \le C r^k \mathbb{E}\|Z_0 - \tilde Z_0\|.
    \end{align*}
    Hence, $Z_k$ converges to the stationary distribution at rate $r^k$. Suppose $Z$ and $\hat Z$ are two stationary solutions driven by the same noise. Then
    \begin{align*}
        Z - \hat Z = A\lst(Z - \hat Z).
    \end{align*}
    Iterating,
    \begin{align*}
        Z - \hat Z = A\lst^k (Z - \hat Z).
    \end{align*}
    Letting $k \to \infty$ and using $A\lst^k \to 0$, we obtain $Z - \hat Z = 0$ almost surely. Thus, the stationary distribution is unique.
\end{proof}

\begin{proof}[\textbf{Proof of Proposition~\ref{prop:sysid}}]
    (a) It suffices to show that $$\norm{Z_{k+1}} \leq \max{\{\norm{Z_k}, B/(1 - \lambda_{\max})\}}.$$
    First, let us assume that $\norm{Z_k} \leq B/(1 - \lambda_{\max})$. Then,
    \begin{align*}
        \norm{Z_{k+1}} &= \norm{A\lst Z_k + w_k} \\
        &\leq \norm{A\lst Z_k} + \norm{w_k} \\
        &\leq B\frac{\lambda_{\max}}{1 - \lambda_{\max}} + B \\
        &= \frac{B}{1 - \lambda_{\max}}.
    \end{align*}
    Now, let $\norm{Z_k} > B/(1 - \lambda_{\max})$.
    \begin{align*}
        \norm{Z_{k+1}} &= \norm{A\lst Z_k + w_k} \\
        &\leq \norm{A\lst Z_k} + \norm{w_k} \\
        &\leq B\frac{\lambda_{\max}}{1 - \lambda_{\max}} + \lambda_{\max} \br{\norm{Z_k} - \frac{B}{1 - \lambda_{\max}}} + B\\
        &= \frac{B}{1 - \lambda_{\max}} + \lambda_{\max} \br{\norm{Z_k} - \frac{B}{1 - \lambda_{\max}}} \\
        &= B + \lambda_{\max} \norm{Z_k} \\
        &< \norm{Z_k}
    \end{align*}
    Hence, we have what we wanted to show.
    
    (b) Rewriting $\cL(A)$ in the following form:
    \begin{align*}
        \cL(A) = \frac{1}{2} \int{z^\top (A - A\ust)^\top (A - A\ust) z ~d\pi(z)}.
    \end{align*}
    Using $\tr(x^\top M x) = \tr(M x x^\top)$,
    \begin{align*}
        \cL(A) &= \frac{1}{2} \tr((A - A\ust)^\top (A - A\ust) \Sigma).
    \end{align*}
    Differentiating, we have $\nabla \cL(A) = (A - A\ust) \Sigma$. Since $\Sigma \succeq \mu_{\min} I$, $(A - A\ust) \Sigma (A - A\ust)^\top \succeq \mu_{\min} (A - A\ust) (A - A\ust)^\top$. See that
    \begin{align*}
        \norm{\nabla \cL(A)}_F^2 &= \norm{(A - A\ust) \Sigma}_F^2 \\
        &= \tr((A - A\ust) \Sigma^2 (A - A\ust)^\top) \\
        &\geq \mu_{\min} \tr((A - A\ust) \Sigma (A - A\ust)^\top) \\
        &= 2\mu_{\min} \cL(A)
    \end{align*}
    (c) From the gradient expression, we have that
    \begin{align*}
        \norm{\nabla \cL(A) - \nabla \cL(A\up)}_F = \norm{(A - A\up) \Sigma}_F \leq \mu_{\max} \norm{A - A\up}_F.
    \end{align*}
    (d) Using the bound on $\norm{Z_k}$, and from the fact that $\mu_{\min} I \prec \Sigma$, we derive the claim as follows.
    \begin{align*}
        \norm{\nabla \tilde{\cL}(A_k,Z_k)}_F^2 &= \norm{(A_k Z_k - A\lst Z_k) Z_k^\top}_F^2 \\
        &= \norm{Z_k Z_k^\top}^2 \norm{A_k - A\lst}_F^2 \\
        &\leq \frac{\norm{Z_k}^4}{\mu_{\min}} \tr\br{(A_k - A\lst)^\top \Sigma (A_k - A\lst)} \\
        &\leq \frac{2 B^4}{(1 - \lambda_{\max})^4 \mu_{\min}} \frac{1}{2} \tr\br{(A_k - A\lst)^\top \Sigma (A_k - A\lst)} \\
        &= \frac{2 B^4}{(1 - \lambda_{\max})^4 \mu_{\min}} \cL(A).
    \end{align*}
    (e) From the bound on $\norm{Z_k}$, the claim follows:
    \begin{align*}
        \norm{w_k Z_k^\top}^2 \leq \norm{w_k}^2 \norm{Z_k}^2 \leq \frac{B^4}{(1 - \lambda_{\max})^2}.
    \end{align*}
\end{proof}
\section{Stepsize-related results}\label{sec:stepsize}

\begin{lemma}\label{lem:zeta_bdd}
    Let the stepsize sequence $\{\alpha_k\}$ be of the form
    \begin{align*}
        \alpha_k = \frac{a}{k+K_0},
    \end{align*}
    and $K_0 \geq \mu a$, then
    \begin{align}
        \zeta_{m,n} &\leq \br{\frac{m + K_0}{n+K_0+1}}^{\mu a}  \label{cmn:ub1} \\
        &\leq e \br{\frac{m + K_0 - 1}{n+K_0+1}}^{\mu a}, \label{cmn:ub2}
    \end{align}
    for every $m \leq n$, where $m,n$ are natural numbers and $\zeta_{m,n}$ is as defined in \eqref{def:zeta}.
\end{lemma}
\begin{proof}
    We use the convention that $\zeta_{m,n} = 1$ if $n < m$. Note the following standard inequality: $1+x \le e^x$. Using this inequality, we obtain that,
    \begin{align*}
        \zeta_{m,n} \leq \exp\br{-\mu \sum_{j=m}^{n} \alpha_j}.
    \end{align*}
    Since $\alpha_j=a/(j+K_0)$ is decreasing in $j$, we have
    \begin{align*}
        -\mu \sum_{j=m}^{n}{\alpha_j} \leq -\mu a \int_{m}^{n+1}{\frac{dy}{y+K_0}} = -\mu a \log\br{\frac{n+K_0+1}{m+K_0}}.
    \end{align*}
    Thus,
    \begin{align*}
        \zeta_{m,n} &\leq  \exp\br{- \mu a \log\br{\frac{n+K_0+1}{m +K_0}}} \notag\\
        &= \br{\frac{m + K_0}{n+K_0+1}}^{\mu a} \\
        &= \br{\frac{m + K_0 - 1}{n+K_0+1}}^{\mu a} \br{\frac{m + K_0}{m+K_0-1}}^{\mu a}.
    \end{align*}
    Noting that $\mu a \leq K_0$ and $m \geq 1$, we have
    \begin{align*}
        \zeta_{m,n} &\leq \br{\frac{m + K_0 - 1}{n+K_0+1}}^{\mu a} \br{1+ \frac{1}{K_0}}^{K_0} \\
        &\leq e \br{\frac{m + K_0 - 1}{n+K_0+1}}^{\mu a}
    \end{align*}
    This concludes the proof.
\end{proof}

\begin{lemma}\label{lem:bddalphazeta}
    Let the stepsize sequence $\{\alpha_k\}$ be of the form
    \begin{align*}
        \alpha_k = \frac{a}{k+K_0},
    \end{align*}
    and $K_0 \geq \mu a$ and $a > 1/\mu$, then, and $\zeta_{m,n}$ be defined as in~\eqref{def:zeta}. Then,
    \begin{enumerate}[(a)]
        \item $\sum_{\ell=0}^{k-1}{\alpha_\ell \zeta_{\ell+1,k-1}} \leq \frac{e - 1}{\mu}$,
        \item $\sum_{\ell=0}^{k-1}{\alpha_\ell^2 \zeta_{\ell+1,k-1}} \leq \frac{e a^2}{\mu a - 1} \frac{1}{k+K_0}$,
        \item $\alpha_{\ell+1} \zeta_{\ell+2,k-1} - \alpha_\ell \zeta_{\ell+1,k-1} \leq 2 \frac{\mu a - 1}{a} \alpha_{\ell+1}^2 \zeta_{\ell+2,k-1}$.
    \end{enumerate}
\end{lemma}
\begin{proof}
    \begin{enumerate}[(a)]
        \item The proof follows from simple algebra as follows:
            \begin{align*}
                \sum_{\ell=0}^{k-1}{\alpha_\ell \zeta_{\ell+1,k-1}} &\stackrel{(a)}{\leq} \frac{e a}{(k+K_0)^{\mu a}}\sum_{\ell=0}^{k-1}{\br{\ell+K_0}^{\mu a - 1}} \notag\\
                &\stackrel{(b)}{\leq} \frac{e a}{(k+K_0)^{\mu a}}\int_{0}^{k}{\br{y+K_0}^{\mu a - 1} dy} \notag\\
                &\leq \frac{e a}{(k+K_0)^{\mu a}} \frac{(k+K_0)^{\mu a} - K_0^{\mu a}}{\mu a} \notag\\
                &\stackrel{(c)}{\leq} \frac{e-1}{\mu}.
            \end{align*}
            $(a)$ follows from Lemma~\ref{lem:zeta_bdd}, $(b)$ follows since $\mu a > 1$, and $(c)$ follows from the condition that $K_0 \geq \mu a$.
        \item 
        \begin{align*}
            \sum_{\ell=0}^{k-1}{\alpha_\ell^2 \zeta_{\ell+1,k-1}} &\stackrel{(a)}{\leq} \frac{e a^2}{(k+K_0)^{\mu a}}\sum_{\ell=0}^{k-1}{\br{\ell+K_0}^{\mu a - 2}} \notag\\
            &\stackrel{(b)}{\leq} \frac{e a^2}{(k+K_0)^{\mu a}}\int_{-1}^{k}{\br{y+K_0}^{\mu a - 2} dy} \notag\\
            &\leq \frac{e a^2}{\mu a - 1} \frac{1}{k+K_0}.
        \end{align*}
        Similar to the last proof, $(a)$ follows from Lemma~\ref{lem:zeta_bdd}, $(b)$ follows since $\mu a > 1$.
        \item The statement follows from the definition:
        \begin{align*}
            \alpha_{\ell+1} \zeta_{\ell+2,k-1} - \alpha_\ell \zeta_{\ell+1,k-1} &= \br{\alpha_{\ell+1} + \mu \alpha_\ell \alpha_{\ell+1} - \alpha_\ell} \zeta_{\ell+2,k-1} \notag \\
            &\leq 2 \frac{\mu a - 1}{a} \alpha_{\ell+1}^2 \zeta_{\ell+2,k-1}.
        \end{align*}
    \end{enumerate}
\end{proof}

\section{Some Useful Results}\label{sec:useful_res}

\begin{lemma}[Azuma-Hoeffding inequality~\texorpdfstring{\cite[p. 36]{wainwright2019high}]}{martin}]\label{lem:ah_ineq}
    Let $X_1, X_2, \ldots$ be a martingale difference sequence with $|X_i| \leq c_i,~\forall i$. Then for all $\eps > 0$ and $n \in \{1,2,\ldots\}$,
    \begin{align}
        \bP\br{\sum_{i = 1}^{n}{X_i} \geq \eps } \leq e^{-\frac{\eps^2}{2\sum_{i = 1}^{n}{c_i^2}}}.
    \end{align}
\end{lemma}

\begin{lemma}\label{lem:ub_grad}
    Let $f:\mathbb{R}^d \to \mathbb{R}$ is $L$-smooth~\eqref{eq:smoothness} and let $f\ust = \inf_{x\in\mathbb{R}^d} f(x)$ be finite. Then for every $x\in\mathbb{R}^d$,
    \begin{align*}
        \|\nabla f(x)\|^2 \leq 2L \br{f(x) - f\ust}.
    \end{align*}
\end{lemma}
\begin{proof}
    The $L$-smoothness of $f$ implies that, for all $x,y \in \bR^d$, $f(y) \le f(x) + \angl{\nabla f(x), y-x} + \frac{L}{2}\|y-x\|^2$~\cite[Lemma~1.2.3]{nesterov2013introductory}. Apply this with $y = x - \tfrac{1}{L}\nabla f(x)$:
    \begin{align*}
        f\br{x - \tfrac{1}{L}\nabla f(x)} &\le f(x) + \angl{\nabla f(x), -\tfrac{1}{L}\nabla f(x)} + \tfrac{L}{2}\norm{-\tfrac{1}{L}\nabla f(x)}^2 \\
        &= f(x) - \frac{1}{2L}\|\nabla f(x)\|^2.
    \end{align*}
    Since $f\ust \le f\!\left(x - \tfrac{1}{L}\nabla f(x)\right)$ by definition of $f\ust$, subtracting yields
    $$f(x) - f\ust \ge f(x) - f\br{x - \tfrac{1}{L}\nabla f(x)} \ge \frac{1}{2L}\|\nabla f(x)\|^2.$$
    This proves the claim.
\end{proof}

\begin{lemma}\label{lem:Delta_const_bd}
Let $c_1,c_2 > 0$, $c_3 \geq 2$ and $\delta \in (0,1)$. Then, for all $y>0$,
\[
\frac{c_1 + c_2 \log\!\big(\tfrac{y}{\delta}\big)}{y + c_3}
\le
\frac{c_1 + c_2 \log\!\big(\tfrac{c_3}{\delta}\big)}{c_3}.
\]
\end{lemma}

\begin{proof}
Set
\[
K \coloneqq \frac{c_1 + c_2 \log\!\big(\tfrac{c_3}{\delta}\big)}{c_3}>0.
\]
The desired inequality is equivalent (since $y+c_3>0$) to
\[
c_1+c_2\log\!\Big(\frac{y}{\delta}\Big)\le K(y+c_3).
\]
Using $Kc_3=c_1+c_2\log\!\big(\tfrac{c_3}{\delta}\big)$, this becomes
\[
0\le Ky + c_2\log\!\Big(\frac{c_3}{y}\Big)\eqqcolon \phi(y).
\]
Now $\phi$ is convex on $(0,\infty)$ because $\phi''(y)=c_2/y^2>0$, hence it attains its global minimum at
$\phi'(y)=0 \iff y_0=c_2/K$. Therefore,
\[
\phi(y)\ge \phi(y_0)=c_2 + c_2\log\!\Big(\frac{K c_3}{c_2}\Big).
\]
Moreover,
\[
\frac{K c_3}{c_2}
= \frac{c_1}{c_2}+\log\!\frac{c_3}{\delta}
> \log(2)
> 1/e,
\]
where the second last inequality uses $c_3 / \delta > 2$. Hence $\phi(y_0) >0$, so $\phi(y)\ge0$ for all $y>0$, proving the claim.
\end{proof}

\begin{lemma}\label{lem:lb_K_0}
    Let $C \geq 1$ and $\delta \in (0,1)$. Then, $K = c_1 \log\br{\frac{2 c_1}{\delta}}$ solves for $K \geq C \log\br{\frac{2K}{\delta}} \br{1 + \frac{\log\br{\frac{2K}{\delta}}}{\log\br{\frac{2}{\delta}}}}$, where
    \begin{align*}
        c_1 = 12 C \log(12 C) + 6 C.
    \end{align*}
\end{lemma}
\begin{proof}
    Let $K = y \log\br{\frac{2y}{\delta}}$. We want to find $y$ such that
    \begin{align*}
        y \log\br{\frac{2y}{\delta}} \geq C \log\br{\frac{2 y}{\delta} \log\br{\frac{2 y}{\delta}}} \br{1 + \frac{\log\br{\frac{2y}{\delta}\log\br{\frac{2 y}{\delta}}}}{\log\br{\frac{2}{\delta}}}}.
    \end{align*}
    Using the facts that $\log\br{\frac{2 y}{\delta}} \leq \frac{2 y}{\delta}$, $\log\br{\frac{2 y}{\delta}}/ \log\br{\frac{2}{\delta}} \leq \log\br{2 y}/ \log\br{2}$ for every $y \geq 1$ and $2/log(2) < 3$, we have
    \begin{align*}
         C \log\br{\frac{2 y}{\delta} \log\br{\frac{2 y}{\delta}}} \br{1 + \frac{\log\br{\frac{2y}{\delta} \log\br{\frac{2 y}{\delta}}}}{\log\br{\frac{2}{\delta}}}} \leq 6 C \br{1 + \log\br{y}} \log\br{\frac{2 y}{\delta}}.
    \end{align*}
    Hence, it suffices to find $y$ such that
    \begin{align*}
        y \geq 6 C \br{1 + \log\br{y}}.
    \end{align*}
    Letting $y = 12 C \log(12 C) + 6 C$
    \begin{align*}
        6 C \br{1 + \log\br{12 C \log(12 C) + 6 C}} &\leq 6 C \log\br{12 C (12 C - 1) + 6 C} + 6C \\
        &\leq 6 C \log\br{(12 C)^2} + 6C \\
        &= 12 C \log\br{12 C} + 6C.
    \end{align*}
\end{proof}

The following lemma is a well-known result on the total variation norm of measures, and follows from \cite[Proposition~D.2.4]{douc2018markov}. 
\begin{lemma}\label{lem:bdd_dotdifLv}
    Let $\mu_1$ and $\mu_2$ be two probability measures on $\cZ$ and let $v$ be an $\bR$-valued bounded function on $\cZ$. Then, the following holds:
    \begin{align*}
        \abs{\int_{\cZ}{(\mu_1 - \mu_2)(\rmd z) v(z) }} \leq \frac{1}{2}\norm{\mu_1 - \mu_2}_{TV} \br{\sup_{z \in \cZ}{v(z)} - \inf_{z \in \cZ}{v(z)}}.
    \end{align*}
    Further, let $u$ be an $\bR^d$-valued bounded function on $\cZ$. Then, 
    \begin{align*}
        \norm{\int_{\cZ}{(\mu_1 - \mu_2)(\rmd z) u(z) }} \leq \norm{\mu_1 - \mu_2}_{TV} \sqrt{d} \sup_{z \in \cZ}{\norm{u(z)}}.
    \end{align*}
\end{lemma}
\begin{proof}
    The first part trivially follows from Equation~D.2.4 in \cite{douc2018markov}. The second part follows from the first part as follows.
    \begin{align*}
        \norm{\int_{\cZ}{(\mu_1 - \mu_2)(\rmd z) u(z) }} &\stackrel{(a)}{\leq} \frac{1}{2}\norm{\mu_1 - \mu_2}_{TV} \br{\sum_{i=1}^{d}{\br{\sup_{z \in \cZ}{u(z)(i)} - \inf_{z \in \cZ}{u(z)(i)}}}}^{1/2} \\
        &\stackrel{(b)}{\leq} \norm{\mu_1 - \mu_2}_{TV} \sqrt{d} \max_{i=1,2,\ldots,d}\flbr{\sup_{z \in \cZ}{\abs{u(z)(i)}}} \\
        &\stackrel{(c)}{\leq} \norm{\mu_1 - \mu_2}_{TV} \sqrt{d} \sup_{z \in \cZ}{\norm{u(z)}},
    \end{align*}
    where $(a)$ follows from the first part of this lemma, $(b)$ follows from the facts that $\norm{\cdot} \leq \sqrt{d} \norm{\cdot}_{\infty}$ and $(c)$ follows from boundedness of $u$ and from the fact that $\norm{\cdot}_{\infty} \leq \norm{\cdot}$.
\end{proof}


\bibliographystyle{siamplain}
\bibliography{refs}
\end{document}